\documentclass[10pt,journal,compsoc]{IEEEtran}
\ifCLASSOPTIONcompsoc
  \usepackage[nocompress]{cite}
\else
  \usepackage{cite}
\fi
\ifCLASSINFOpdf
   \usepackage[pdftex]{graphicx}
\else
   \usepackage[dvips]{graphicx}
\fi
\usepackage{amsmath}
\usepackage{algorithmic}
\ifCLASSOPTIONcompsoc
 \usepackage[caption=false,font=footnotesize,labelfont=sf,textfont=sf]{subfig}
\else
 \usepackage[caption=false,font=footnotesize]{subfig}
\fi

\usepackage{stfloats}
\usepackage{adjustbox}
\usepackage{makecell}
\usepackage{multirow}
\usepackage[ruled,vlined,linesnumbered]{algorithm2e}
\usepackage{enumitem}
\usepackage{amsmath,amsthm,amsfonts}
\usepackage{color}
\usepackage{xurl}

\newtheorem{proposition}{Proposition}

\begin{document}
%
\title{Multi-grained Semantics-aware Graph Neural Networks}

%
%
%
%

\author{Zhiqiang Zhong,
        Cheng-Te Li,
        and~Jun Pang 
\IEEEcompsocitemizethanks{
\IEEEcompsocthanksitem Zhiqiang Zhong is with the Faculty of Science, Technology and Medicine, 
University of Luxembourg, Esch-sur-Alzette, Luxembourg. \protect\\
E-mail: zhiqiang.zhong@uni.lu

\IEEEcompsocthanksitem Cheng-Te Li is with the Institute of Data Science and the Department of Statistics, National Cheng Kung University, Tainan, Taiwan.\protect\\
E-mail: chengte@mail.ncku.edu.tw

\IEEEcompsocthanksitem Jun Pang is with the Faculty of Science, Technology and Medicine, 
and the Interdisciplinary Centre for Security, Reliability and Trust,
University of Luxembourg, Esch-sur-Alzette, Luxembourg. \protect\\
E-mail: jun.pang@uni.lu
}
}
\markboth{Journal of \LaTeX\ Class Files,~Vol.~14, No.~8, August~2015}%
{Shell \MakeLowercase{\textit{et al.}}: Bare Demo of IEEEtran.cls for Computer Society Journals}
%



\IEEEtitleabstractindextext{%
\begin{abstract}
Graph Neural Networks (GNNs) are powerful techniques in representation learning for graphs and have been increasingly deployed in a multitude of different applications that involve node- and graph-wise tasks. 
Most existing studies solve either the node-wise task or the graph-wise task independently while they are inherently correlated. 
This work proposes a unified model, AdamGNN, to interactively learn node and graph representations in a mutual-optimisation manner. 
Compared with existing GNN models and graph pooling methods, AdamGNN enhances the node representation with the learned multi-grained semantics and avoids losing node features and graph structure information during pooling. 
Specifically, a differentiable pooling operator is proposed to adaptively generate a multi-grained structure that involves meso- and macro-level semantic information in the graph. 
We also devise the unpooling operator and the \textit{flyback} aggregator in AdamGNN to better leverage the multi-grained semantics to enhance node representations. 
The updated node representations can further adjust the graph representation in the next iteration. 
Experiments on $14$ real-world graph datasets show that AdamGNN can significantly outperform $17$ competing models on both node- and graph-wise tasks. 
The ablation studies confirm the effectiveness of AdamGNN's components, and the last empirical analysis further reveals the ingenious ability of AdamGNN in capturing long-range interactions.

\end{abstract}

}

\maketitle

\IEEEdisplaynontitleabstractindextext

%
\IEEEpeerreviewmaketitle


\IEEEraisesectionheading{\section{introduction} 
\label{sec:introduction}}
\IEEEPARstart{I}{n} many real-world applications, such as social networks, recommendation systems, and biological networks, data can be naturally organised as graphs~\cite{HYL17}. 
Nevertheless, working with this powerful node and graph representations remains a challenge, since it requires integrating the rich inherent features and complex structural information. 
To address this challenge, Graph Neural Networks (GNNs), which generalise deep neural networks to graph-structured data, have drawn remarkable attention from academia and industry, and achieve state-of-the-art performances in a multitude of applications~\cite{WPCLZY20,ZCZ20}. 
The current literature on GNNs can be used for tasks with two categories. 
One is to learn node representations to perform tasks such as link prediction~\cite{ZC18} and node classification~\cite{KW17,XHLJ19}.
The other is to learn graph representations for tasks, such as graph classification~\cite{YYMRHL18,GJ19,YJ20}.

On node-wise tasks, existing GNN models on learning node representations rely on a similar methodology that utilises a GNN layer to aggregate the sampled neighbouring nodes' features in a number of iterations, via non-linear transformation and aggregation functions.
Its effectiveness has been widely proved, however, a major limitation of these GNN models is that they are inherently \textit{flat} as they only propagate information across the observed edges in the original graph.
Thus, they lack the capacity to encode features in the high-order neighbourhood in the graphs~\cite{YYL19,BKMPRS20}.
For example, in an academic collaboration network, \textit{flat} GNN models could capture the micro semantic (e.g., co-authorships) between authors, but neglect their macro semantics (e.g., belonging to different research institutes).

On the other hand, graph classification is to predict the label associated with an entire graph by utilising the given graph structure and initial node features.
Nevertheless, existing GNNs for graph classification are unable to learn graph representations in a \textit{multi-grained} manner, which is crucial to encode better meso- and macro-level graph semantics hidden in the graphs for many practical applications such as drug design~\cite{SXZZZT20} and program analysis~\cite{LFPZ04}.
To remedy this limitation, novel pooling approaches have been proposed, where sets of nodes are recursively aggregated to form super nodes in the pooled graph.  
\textsc{DiffPool}~\cite{YYMRHL18} is a differentiable pooling operator but its assignment matrix is too \textit{dense}~\cite{CVJKL18} to apply on large graphs.
\textsc{g-U-Net}~\cite{GJ19}, \textsc{SagPool}~\cite{LLK19}, GXN~\cite{LCZT20} and ASAP~\cite{RST20} are four recently proposed methods that adopt the Top-$k$ selection strategy to address the sparsity concerns of \textsc{DiffPool}.
They score nodes based on a learnable projection vector and select a fraction of high scoring nodes as super nodes. 
However, the pre-defined pooling ratio limits these models' adaptivity on graphs with different sizes, and the Top-$k$ selection may easily lose important node features or graph structure by simply ignoring low scoring nodes.
As shown in Figure~\ref{fig:topk_selection} (Section~\ref{sec:related_work}), different numbers of $k$ will significantly affect the number of covered nodes of super nodes in the pooled graph, which means the important nodes' features could get lost during the trivial pooling strategy. 
The hyper-parameter $k$ is also crucial for the final performance~\cite{GJ19}, thus reduces their convenience in applications.

In the end, we argue that node- and graph-wise tasks are inherently correlated with one another. 
That said, node representations form graph representation, and graph representation can provide node representations with meso/macro-level semantic information in the graph.
Joint modelling with node- and graph-wise tasks will allow GNNs to overcome the limitation of \textit{flat} propagation mode in capturing multi-grained semantics,
and the enriched node representation could further ameliorate the graph representation. 
However, to the best of our knowledge, none of the existing work simultaneously exploit node- and graph-wise tasks, along with capturing multi-grained semantics hidden in the graph, to learn representations of nodes and the graph.

In this work, we propose a novel framework, \textit{\underline{Ada}ptive \underline{M}ulti-grained Graph Neural Networks} (AdamGNN), which integrates graph convolution, adaptive pooling and unpooling operations into one framework to generate both node and graph level representations interactively.
Unlike the above-mentioned GNN models, we treat node and graph representation learning tasks in a unified framework so that they can collectively optimise each other during training.
In modelling multi-grained semantics, the adaptive pooling and unpooling operators preserve the important node features and hierarchical structural features.
More concretely, as shown in Figure~\ref{fig:architesture_AdamGNN}-(a),
we employ {\it (i)} an adaptive graph pooling (AGP) operators to generate a multi-grained structure based on the derived primary node representations by a GNN layer, 
{\it (ii)} graph unpooling (GUP) operators to further distribute the explored meso- and macro-level semantics to the corresponding nodes of the original graph, 
and {\it (iii)} a \textit{flyback} mechanism to integrate all received multi-grained semantic messages as the evolved node representations.
Besides, the attention-enhanced \textit{flyback} aggregator provides a reasonable explanation of the importance of messages from different grains.
Experimental results reveal the effectiveness of AdamGNN, and the ablation and empirical studies confirm the effectiveness and flexibility of different components in AdamGNN.
At last, through case studies, AdamGNN is shown to highlight variant-range node interactions in different graph datasets.

Our contributions can be summarised as follows.
(1) We propose a novel framework, AdamGNN\footnote{Code and data are available at: \url{https://github.com/zhiqiangzhongddu/AdamGNN}}, 
that adaptively integrates multi-grained semantics into node representations and achieves mutual optimisation between node-wise and graph-wise tasks in one unified process.
(2) An adaptive and efficient pooling operator is devised in AdamGNN to generate the multi-grained structure without introducing any hyper-parameters.
(3) An attention-based flyback aggregation can provide model explainability on how different grains benefit the prediction tasks.
(4) Extensive experiments on $14$ real-world datasets demonstrate the promising performance of AdamGNN, along with providing insightful explanations with case studies.


\section{Related Work} 
\label{sec:related_work}
\textbf{Graph neural networks}.
The existing GNN models can be generally categorised into spectral and spatial approaches.
The spectral approach utilises the Fourier transformation to define convolution operation in the graph domain~\cite{BZSL14}.
However, its incurred heavy computation cost hinders it from being applied to large-scale graphs.
Later on, a series of spatial models drawn remarkable attention due to their effectiveness and efficiency in node-wise tasks~\cite{KW17,HYL17,VCCRLB18,XHLJ19,CCBLV20,LYW21,TC21,CWHDL20,FZDHLXYKT20}, such as link prediction, node classification and node clustering.
They mainly rely on the \textit{flat} message-passing mechanism 
that defines convolution by iteratively aggregating messages from the neighbouring nodes. 
Recent studies have proved that the spatial approach is a special form of Laplacian smoothing and is limited to summarising each node's local information~\cite{LHW18,CLLLZS20}. 
Besides, they are either unable to capture global information
or incapable of aggregating messages in a multi-grained manner to support graph classification tasks.

\begin{figure}[!t]
\centering
\includegraphics[width=.75\linewidth]{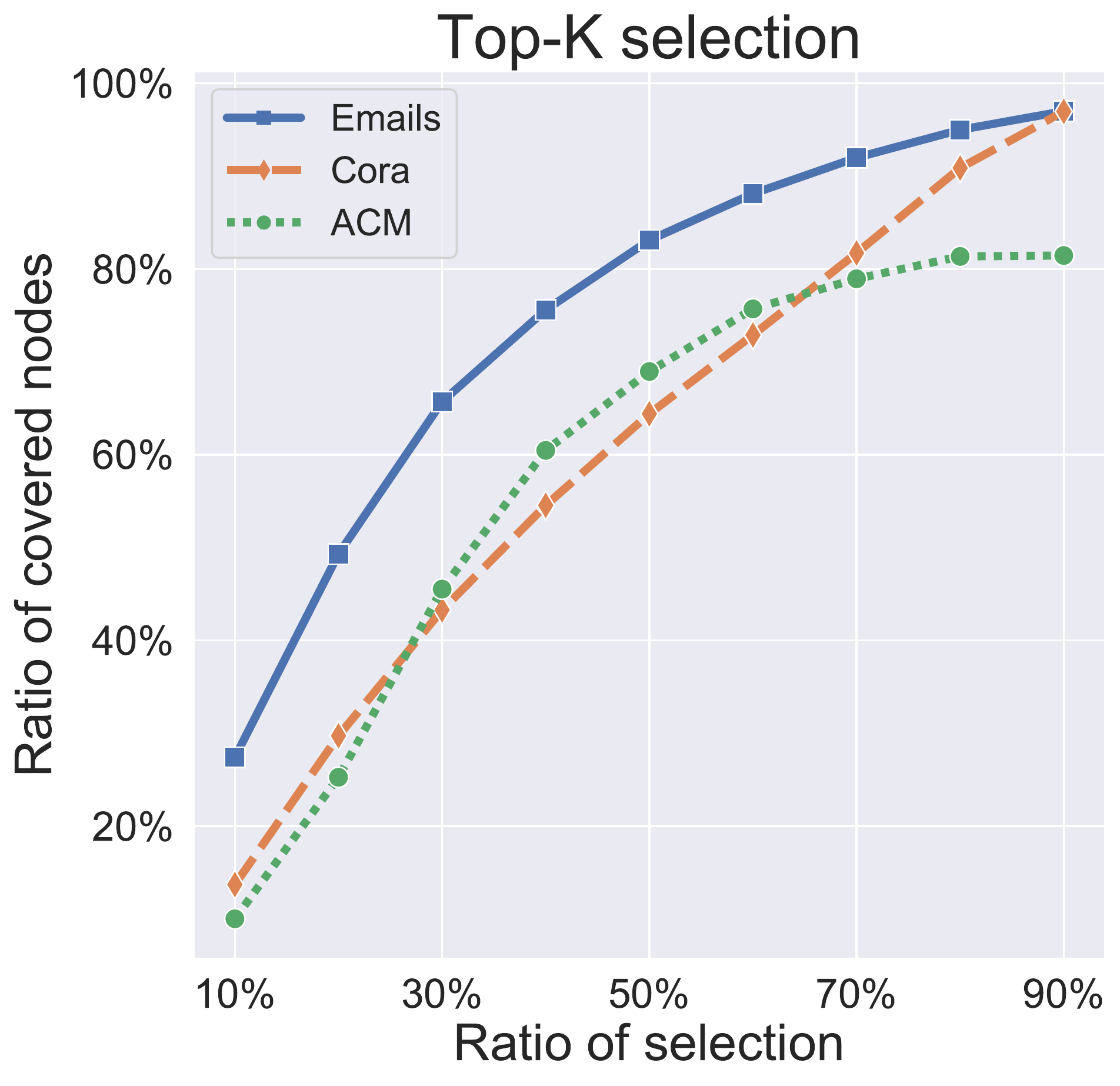}
\caption{
Ratio of covered nodes with various selection ratio.
}
\label{fig:topk_selection}
\vspace{-3mm}
\end{figure}

\smallskip\noindent
\textbf{Graph pooling}.
Pooling operation overcomes GNN's weakness in generating graph-level representation by recursively merge sets of nodes to form super nodes in the pooled graph.
\textsc{DiffPool}~\cite{YYMRHL18} is a differentiable pooling operator, which learns a soft assign matrix that maps each node to a set of clusters, treated as super nodes.
Since this assignment is rather \textit{dense} that incurs high computation cost, it is not scalable for large graphs~\cite{CVJKL18}.
Following this direction, a Top-$k$ selection based pooling layer (\textsc{g-U-Net}) is proposed to select important nodes from the original graph to build a pooled graph~\cite{GJ19}.
\textsc{SagPool}~\cite{LLK19} and ASAP~\cite{RST20} further use attention and self-attention for cluster assignment, GXN~\cite{LCZT20} uses mutual information estimation for super node selection. 
They address the problem of sparsity in \textsc{DiffPool}, however, such a manual-defined hyper-parameter $k$ is quite sensitive to the final performance~\cite{GJ19}, thus limits the adaptivity of these models on graphs of different sizes. 
In addition, as shown in Figure~\ref{fig:topk_selection}, different $k$ values will significantly affect the number of covered nodes in the graph, which means the important node features could get lost during the trivial pooling operation. 
Note that nodes covered by a super node refer to nodes involved in the super node's aggregation tree.

Recently, \textsc{EigenPool}~\cite{MWAT19} proposes a pooling operator based on graph Fourier which does not rely on the Top-$k$ selection strategy, \textsc{StructPool}~\cite{YJ20} designs strategies to involve both node and graph structures, and includes conditional random fields technique to ameliorate the cluster assignment.
However, \textsc{StructPool} treats the graph assignment as a \textit{dense} clustering problem, which gives rise to a high computation complexity as in \textsc{DiffPool}.

\begin{table}[!t]
\caption{Model comparison from various aspects: Node-wise Task (NT), Graph-wise Task (GT), Pooling and/or Unpooling (P/U), Adaptive Pooling (AP), Efficient Pooling (EP), Multi-grained Explanation (ME).
}
\label{table:comparison_diff_models}
\centering
\begin{tabular}{l | c | c | c | c | c | c}
\hline
& NT & GT & P/U & AP & EP & ME     \\
\hline
\hline
GCN~\cite{KW17}                     & $\checkmark$   & &              &           &           \\
\hline
GraphSAGE~\cite{HYL17}              & $\checkmark$   & &             &           &           \\
\hline
GAT~\cite{VCCRLB18}                 & $\checkmark$   & &              &           &           \\
\hline
GIN~\cite{XHLJ19}                   & $\checkmark$   & $\checkmark$ &              &           &           \\
\hline
PNA~\cite{CCBLV20}                  & $\checkmark$   & $\checkmark$ &              &           &           \\
\hline
\textsc{GCNII}~\cite{CWHDL20}       & $\checkmark$   & &              &           &           \\
\hline
\textsc{GRAND}~\cite{FZDHLXYKT20}   & $\checkmark$   & &              &           &           \\
\hline
\textsc{DiffPool}~\cite{YYMRHL18}   &           & $\checkmark$ & P &        &           &           \\
\hline
\textsc{g-U-Net}~\cite{GJ19}        & $\checkmark$   & $\checkmark$ & P, U &   &  $\checkmark$  &           \\
\hline
\textsc{SagPool}~\cite{LLK19}       &           & $\checkmark$ & P &  & $\checkmark$   &           \\
\hline
\textsc{EigenPool}~\cite{MWAT19}    &           & $\checkmark$ & P & $\checkmark$     & $\checkmark$   &           \\
\hline
GXN~\cite{LCZT20}        & $\checkmark$   & $\checkmark$ & P, U &   &  $\checkmark$  &           \\
\hline
\textsc{StructPool}~\cite{YJ20}     &           & $\checkmark$ & P &   &           &           \\
\hline
ASAP~\cite{RST20}                   &           & $\checkmark$ & P &  & $\checkmark$   &           \\
\hline
\textbf{AdamGNN}                             & $\checkmark$   & $\checkmark$ & P, U &  $\checkmark$      & $\checkmark$   & $\checkmark$   \\
\hline
\end{tabular}
\vspace{-3mm}
\end{table}

\smallskip\noindent
\textbf{Discussion}.
Table~\ref{table:comparison_diff_models} summarises the key advantages of the proposed AdamGNN and compares it with a number of state-of-the-art methods.
Among the existing GNN models, \textsc{g-U-Net}, GXN and AdamGNN support both node- and graph-level tasks.
However, {\it (i)} the Top-\textit{k} selection strategy of \textsc{g-U-Net} and GXN introduces a new hyper-parameter and may lose important node features or graph structure;
{\it (ii)} \textsc{g-U-Net} and GXN generate super graph only with $k$ selected super nodes, which ruins multi-grained semantics of the original graph;
{\it (iii)} \textsc{g-U-Net} does not support mini-batch because it needs to compute scores of all nodes in one big batch to select super nodes;
{\it (iv)} GXN requires positive and negative node sets sampling which introduces additional operation and the random sampling strategy brings unmanageable uncertainty on model performance~\cite{YDZYZT20}.
On the contrary, AdamGNN is a unified framework that adaptively integrates multi-grained semantics into node representations, and achieves a mutual optimisation between node- and graph-wise tasks. 
Besides, AdamGNN supports efficient mini-batch pooling and unpooling, and also provides model explanation via the multi-grained semantics.
Therefore, we believe AdamGNN's framework is more general, effective and scalable than G-U-Net. 



\section{Proposed Approach} 
\label{sec:proposed_approach}
\subsection{Preliminaries}
\label{subsec:preliminaries}
A graph with $n$ nodes can be formally represented as $\mathcal{G}=(\mathcal{V}, \mathcal{E}, \mathbf{X})$,
where $\mathcal{V}=\{v_{1}, \dots, v_{n} \}$ is the node set, $\mathcal{E} \subseteq \mathcal{V} \times \mathcal{V}$ denotes the set of edges, and $\mathbf{X} \in \mathbb{R}^{n \times \pi}$ represents nodes' features ($\pi$ is the dimensionality of node features).
Besides, $\mathcal{V}$ and $\mathcal{E}$ can be summarised in adjacency matrix $\mathbf{A} \in \{0, 1\}^{n \times n}$.

For node-wise tasks, the goal is to learn a mapping function $f_{n}: \mathcal{G} \to \mathbf{Z}$, 
where $\mathbf{Z} \in \mathbb{R}^{d}$, and each row $\mathbf{z}_{i}\in \mathbf{Z}$ corresponds to node $v_i$'s representation.
For graph-wise tasks, similarly it aims to learn a mapping $f_{g}: \mathcal{D} \to \mathbf{Z}$, where $\mathcal{D} = \{\mathcal{G}_{1}, \mathcal{G}_{2}, \dots\}$ is a set of graphs,
each row $\mathbf{z}_{i} \in \mathbf{Z}$ corresponds to the graph $\mathcal{G}_{i}$'s representation. 
The mapping function's effectiveness $f_{n}$ and $f_{g}$ is evaluated by applying $\mathbf{Z}$ to different downstream tasks.

\smallskip\noindent
\textbf{Primary node representation}.
We use Graph Convolution Network (GCN)~\cite{KW17} as an example primary GNN encoder to obtain the node representation:
\begin{equation}
\label{eq:gcn}
    \mathbf{H}^{(\ell+1)} = \mathrm{ReLU}(\hat{\mathbf{D}}^{-\frac{1}{2}} \hat{\mathbf{A}} \hat{\mathbf{D}}^{\frac{1}{2}} \mathbf{H}^{(\ell)} \mathbf{W}^{(\ell)}),
\end{equation}
where $\hat{\mathbf{A}}= \mathbf{A} + \mathbf{I}$, $\hat{\mathbf{D}}=\sum_{j} \hat{\mathbf{A}}_{ij}$ 
and $\mathbf{W}^{(\ell)} \in \mathbb{R}^{n \times d}$ is a trainable weight matrix for layer $\ell$.
$\mathbf{H}^{(\ell)}$ is the generated node representation of layer $\ell$ which is defined as the primary node representations $\mathbf{H}=\mathbf{H}^{(\ell)}$.
Node representations are generated based on each target node's local neighbours, which are aggregated via learning based on the adjacency matrix $\mathbf{A}$.
GCN cannot capture meso/macro-level knowledge, even with stacking multiple layers. 
Hence we term such generated node representations as primary node representations. 

\begin{figure*}[!t]
\centering
\includegraphics[width=.6\linewidth]{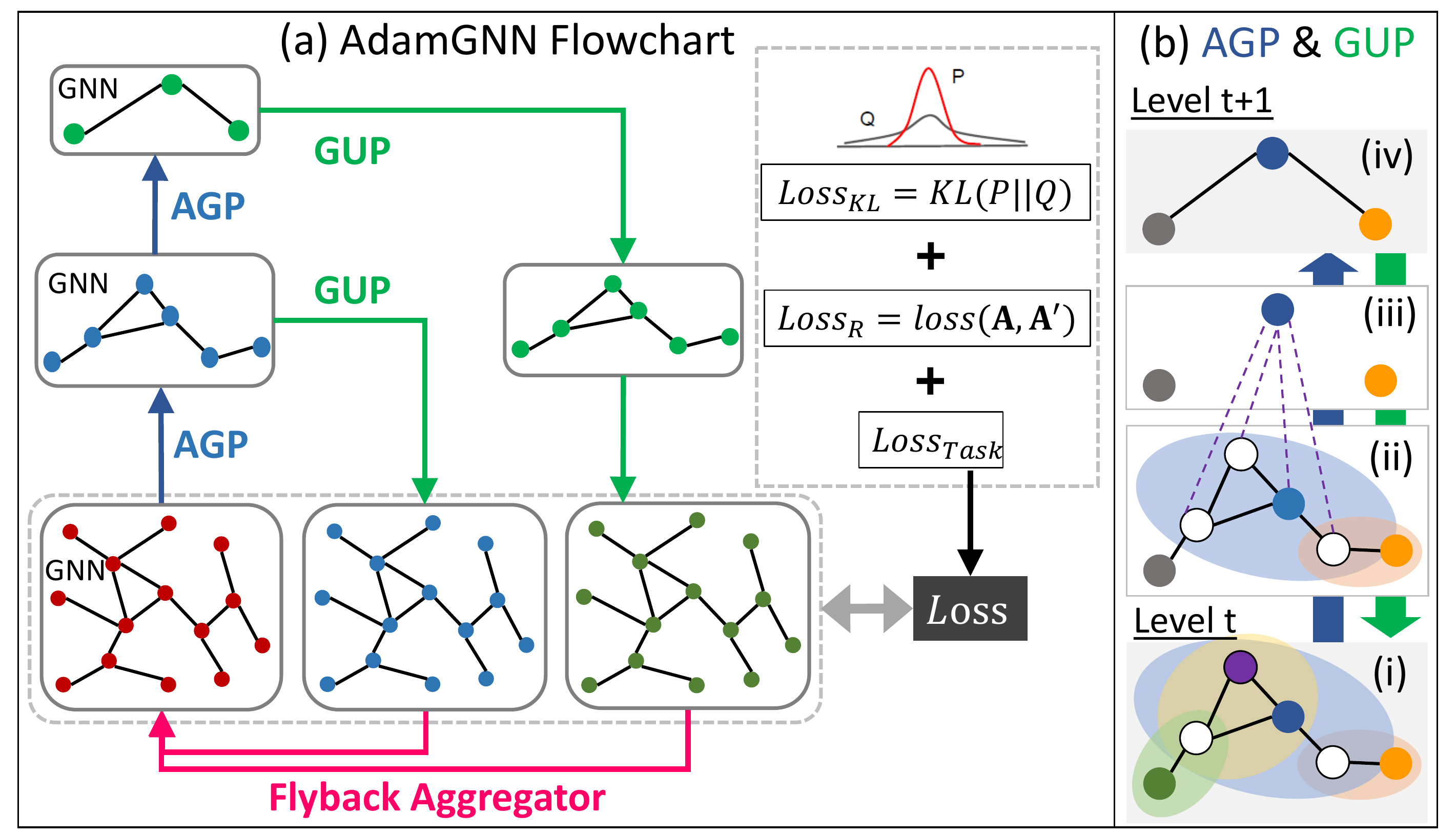}
\caption{
(a) An illustration of AdamGNN with 3 levels.
AGP: adaptive graph pooling,
GUP: graph unpooling.
(b) An example of performing adaptive graph pooling on a graph:
(i) ego-network formation,
(ii-iii) super node generation,
(iv) maintaining super graph connectivity.
}
\label{fig:architesture_AdamGNN}
\vspace{-3mm}
\end{figure*}

\subsection{Adaptive Graph Pooling for Multi-grained Structure Generation}
\label{subsec:adaptive_graph_pooling}
The proposed model, AdamGNN, adaptively generates a multi-grained structure to realise the collective optimisation between the node and graph level tasks within one unified framework.
The key idea is that applying an adaptive graph pooling operator to present the multi-grained semantics of $\mathcal{G}$ explicitly and improve the node representation generation with the derived meso/macro information. 
While AdamGNN is usually performed under multiple levels of granularity ($T$ different grains), in this section, we present how level $t$'s super graph is adaptively constructed based on graph of level $t\!-\!1$, i.e., $\mathcal{G}_{t\!-\!1} = (\mathcal{V}_{t\!-\!1}, \mathcal{E}_{t\!-\!1}, \mathbf{X}_{t\!-\!1})$.

\smallskip\noindent
\textbf{Ego-network formation}.
We initially consider the graph pooling as an ego-network selection problem, i.e., each ego node can determine whether to aggregate its local neighbours to form a super node, resolving the \textit{dense} issue of \textsc{DiffPool} by avoiding to use a dense assignment matrix.
As shown in Figure~\ref{fig:architesture_AdamGNN}-(b)-(i),
each ego-network $c_{\lambda}$ contains the ego and its local neighbours $\mathcal{N}^{\lambda}_{i}$ within $\lambda$-hops., 
i.e., $\mathcal{N}^{\lambda}_{i} = \{v_{j} \mid \, d(v_{i}, v_{j}) \leq \lambda \}$, 
where $d(v_{i}, v_{j})$ means the shortest-path length between $v_{i}$ and $v_{j}$.
Thus an ego-network can be formally presented as:
$c_{\lambda}(v_{i}) = \{v_j \mid \forall v_j \in \mathcal{N}^{\lambda}_{i} \}$, and a set of ego-networks $\mathcal{C}_{\lambda} = \{c_{\lambda}(v_{1}), \dots, c_{\lambda}(v_{n}) \}$ can be generated from $\mathcal{G}$.
We will investigate the impact of the ego-network size in the ablation studies of  Section~\ref{subsec:experimental_evaluation_and_ablation_study}

\smallskip\noindent
\textbf{Super node determination}.
Given $\mathcal{G}$ with $n$ nodes has $n$ ego-networks, forming a super graph with all ego-networks will blur the useful multi-grained semantics and lead to a high computation cost.
We select a fraction of ego-networks from $\mathcal{C}_{\lambda}$ to organise the multi-grained semantics of $\mathcal{G}$.
We make the selection based on a closeness score $\phi_{i}$ that evaluates the representativeness of the ego $v_{i}$ to its local neighbours $v_{j} \in c_{\lambda}(v_{i})$.
We first create a function to calculate the closeness score $\phi_{ij}$ between $v_{i}$ and $v_{j}$:
\begin{equation} 
\label{eq:fiscore}
    \begin{aligned}
    \phi_{ij} = & ~f^{\it non}_{\phi}(v_{i}, v_{j}) \times f^{\it lin}_{\phi}(v_{i}, v_{j}) \\
    = & \mathrm{~Softmax}(\overrightarrow{\mathbf{a}}^{\mathrm{T}} \, \sigma(\mathbf{W}\mathbf{H}[j] \parallel \mathbf{W}\mathbf{H}[i])) \times \\ &  \mathrm{~Sigmoid}(\mathbf{H}^{\mathrm{T}}[j] \cdot \mathbf{H}[i]),
    \end{aligned}
\end{equation}
where $\overrightarrow{\mathbf{a}} \in \mathbb{R}^{2\pi}$ is the weight vector, $\parallel$ is the concatenation operator and $\sigma$ is an activation function ($\mathrm{LeakReLU}$) to avoid the vanishing gradient problem~\cite{H98} during the model training process. 
$f^{\it non}_{\phi}(v_{i}, v_{j}) = \frac{\exp(\overrightarrow{\mathbf{a}}^{\mathrm{T}} \, \sigma(\mathbf{W} \mathbf{H}[j] \parallel \mathbf{W} \mathbf{H}[i]))}{\sum_{v_{r} \in \mathcal{N}^{\lambda}_{j}}\exp(\overrightarrow{\mathbf{a}}^{\mathrm{T}} \, \sigma(\mathbf{W}\mathbf{H}[j] \parallel \mathbf{W}\mathbf{H}[r]))}$ calculates one component of $\phi_{ij}$ considering the \textit{non-linear} relationship between node $v_{j}$'s and ego $v_{i}$' representations, and its output lies in $(0, 1)$ as a valid probability for ego-network selection.
Meanwhile, we further add another component $f^{\it lin}_{\phi}(v_{i}, v_{j})=\mathrm{Sigmoid}(\mathbf{H}^{\mathrm{T}}[j] \cdot \mathbf{H}[i])$ to supercharge $\phi_{ij}$ with the \textit{linearity} between node $v_{j}$ and ego $v_{i}$ to capture more comprehensive information.
Consequently, nodes have similar features and structure information to the ego will contribute to higher closeness scores.
In the end, we produce the closeness score of $c_{\lambda}(v_{i})$ as:
\begin{equation}
\label{eq:sum_fiscore}
    \phi_{i} = \frac{1}{\vert\mathcal{N}^{\lambda}_{i}\vert} \sum_{v_{j} \in \mathcal{N}^{\lambda}_{i}} \phi_{ij},
\end{equation}
where $\vert\mathcal{N}^{\lambda}_{i}\vert$ indicates the number of nodes in $\mathcal{N}^{\lambda}_{i}$.

After obtaining ego-networks' closeness scores, we propose an adaptive approach to select a fraction of ego-networks to form super nodes without pre-defined hyper-parameters (cf. the Top-$k$ selection strategy~\cite{GJ19}).
Our key intuition is that a high diameter ego-network could be composed of multiple low diameter ego-networks.
Therefore, we intend first to find these low diameter ego-networks,
then recursively aggregate them to form a super node that contains these ego-networks.
Specifically, we form ego-networks by selecting a fraction of egos $\hat{N}_{p}$ as:
$\hat{N}_{p}=\{v_{i} \mid  \phi_{i} > \phi_{j}, \,  \forall v_{j} \in \mathcal{N}^{1}_{i}\}$,
where $\mathcal{N}^{1}_{i}$ means the neighbour nodes of node $v_{i}$ within the first hop.
Note that each node may belong to various ego-networks since they may play different roles in different groups.
Therefore, we allow overlapping between different selected ego-networks
and utilise $\mathcal{N}^{1}_{i}$ instead of $\mathcal{N}^{\lambda}_{i}$.
If we adopt $\mathcal{N}^{\lambda}_{i}$ here, 
the selected ego node $v_{i}$ cannot be involved in other ego-networks anymore.
Following this, we can select a fraction of ego-networks to form super nodes at granularity level $t$.

\begin{proposition}
\label{prop:1}
Let $\mathcal{G}$ be a connected graph with $n$ nodes, 
and a total number of $n$ ego-networks can be formed from the graph $\mathcal{G}$, i.e., $\mathcal{C}_{\lambda} = \{c_{\lambda}(v_{1}), c_{\lambda}(v_{2}), \dots, c_{\lambda}(v_{n}) \}$.
Each ego-network $c_{\lambda}(v_{i})$ will be assigned with a closeness score $\phi_{i}$.
Then, there exist at least one ego-network $c_{\lambda}(v_{i})$ that satisfies $\phi_{i} > \phi_{j}, \,  \forall v_{j} \in \mathcal{N}^{1}_{i}$.
\end{proposition}

\begin{proof}[Proof of Proposition~\ref{prop:1}]
For $\mathcal{G}=(\mathcal{V}, \mathcal{E}, \mathbf{X})$ with $n$ nodes. 
$n$ ego-networks can be generated by following the procedures, $c_{\lambda}(v_{i}) = \{j \mid \forall j \in \mathcal{N}^{\lambda}_{i} \}$, and each ego-network will be given a closeness score $\phi_{i}$ as Equation~\ref{eq:sum_fiscore}.
We assume that these cluster closeness scores are not all the same, thus there exists at least one maximum $\phi_{max}$.
Hence, the clusters with closeness score $\phi_{max}$ satisfy the requirements of ego-network selection requirement that:
\begin{equation}
    \phi_{max} > \phi_{j}, \, \, \,  \forall v_j \in \mathcal{N}^{1}_{max},
\end{equation}
where $\mathcal{N}^{1}_{max} = \{v_{j} \mid if \, d(v_{i}, v_{j}) = 1 \}$.
So, for any connected $\mathcal{G}$ with $n$ nodes, there exists at least one cluster satisfies the requirements of our ego-network selection approach.
\end{proof}

Proposition~\ref{prop:1} ensures that our super node determination method can find at least one ego-network to generate a super graph for any graph.
It guarantees the generality of our strategy.

Meanwhile, we would also retain nodes that do not belong to any selected ego-networks, denoted as $\hat{N}_{r}$, to maintain the graph structure:
$\hat{N}_{r}=\{v_{j} \mid v_{j} \notin c_{\lambda}(v_{i}), \forall v_{i} \in \hat{N}_{p} \}$. 
In this way, a super node formation matrix $\mathbf{S}_{t} \in \mathbb{R}^{n \times (\vert\hat{N}_{p}\vert+\vert\hat{N}_{r}\vert) }$ can be formed, where $(\vert\hat{N}_{p}\vert+\vert\hat{N}_{r}\vert)$ is number of nodes of the generated super graph, rows of $\mathbf{S}_{t}$ corresponds to the $n$ nodes of $\mathcal{G}_{t\!-\!1}$, and columns of $\mathbf{S}_{t}$ corresponds to the selected ego-networks ($\hat{N}_{p}$) plus the remaining nodes ($\hat{N}_{r}$).
We have $\mathbf{S}_{t}[i,j]=\phi_{ij}$ if node $v_{j}$ belongs to the selected ego-network $c_{\lambda}(v_{i})$ and $\mathbf{S}_{t}[i,j]=1$ if node $v_{j}$ is a remaining node corresponds to node $v_{i}$ in the super graph; otherwise $\mathbf{S}_{t}[i,j]=0$.
The weighted super node formation matrix $\mathbf{S}_{t}$ can better maintain the relation between different super nodes in the pooled graph.

\smallskip\noindent
\textbf{Maintaining super graph connectivity}.
After selecting the ego-networks and retaining nodes in level $t\!-\!1$, as shown in Figure~\ref{fig:architesture_AdamGNN}-(b)-(iii-iv), we construct the new adjacent matrix $\mathbf{A}_{t}$ for the super graph using $\hat{\mathbf{A}}_{t\!-\!1}$ and $\mathbf{S}_{t}$ as follows:
$\mathbf{A}_{t} = \mathbf{S}_{t}^{\mathrm{T}} \hat{\mathbf{A}}_{t\!-\!1} \mathbf{S}_{t}$.
This formula makes any two super nodes connected if they share any common nodes or any two nodes are already neighbours in $\mathcal{G}_{t\!-\!1}$.
In addition, $\mathbf{A}_{t}$ will retain the edge weights passed by $\mathbf{S}_{t\!-\!1}$
that involves the relation weights between super nodes.
Eventually, we obtain a generated super graph $\mathcal{G}_{t}$ at granularity level $t$.

\smallskip\noindent
\textbf{Super node feature initialisation}.
All nodes in the super graph $\mathcal{G}_{t}$ need initial feature vectors to support the graph convolution operation.
Recall that we have the closeness score as calculated in Eq.~\ref{eq:fiscore}, between node $v_{j}$ to the ego $v_{i}$.
However, this is not equivalent to the contribution of node $v_{j}$'s feature to the super node feature, since we need to compare the relationship strength between ego $v_{i}$ and $v_{j}$ with the relation between other $v_{r} \in c_{\lambda}(v_{i})$.
Therefore, we further propose a super node feature initialisation method through a self-attention mechanism~\cite{VCCRLB18}. 
Specifically, it can be described as:
\begin{equation}
    \mathbf{X}_{t}[i] = \mathbf{H}_{t\!-\!1}[i] + \sum_{v_{j} \in c_{\lambda}(v_{i})\setminus v_{i}} \alpha_{ij} \mathbf{H}_{t\!-\!1}[j],
\end{equation}
where $\mathbf{H}_{t\!-\!1}$ is the generated node representation by the ($t\!-\!1$)-th primary GNN layer, i.e., at level $t\!-\!1$ with a similar method as Equation~\ref{eq:gcn}, 
$\alpha_{ij}$ describes the importance of node $v_{j}$ to the initial feature of $c_{\lambda}(v_{i})$ at level $t$.
And $\alpha_{ij}$ can be learned as follows:
$\alpha_{ij} = \frac{\exp(\overrightarrow{\mathbf{a}_{1}}^{\mathrm{T}} \, \sigma(\mathbf{W} (\phi_{ij} \mathbf{H}_{t-1}[j]) \parallel \mathbf{H}_{t-1}[i]))}{\sum_{v_{r} \in c_{\lambda}(v_{i})}\exp(\overrightarrow{\mathbf{a}_{1}}^{\mathrm{T}} \, \sigma(\mathbf{W} (\phi_{ir} \mathbf{H}_{t-1}[r]) \parallel \mathbf{H}_{t-1}[i]))})$ 
where $\overrightarrow{\mathbf{a}_{1}} \in \mathbb{R}^{2\pi}$ is the weight vector. 
For the remaining nodes $\hat{N}_{r}$ that do not belong to any super nodes, we keep their representations of $\mathbf{H}_{t\!-\!1}$ as initial node features. 

\subsection{Graph Unpooling}
\label{subsec:graph_unpooling}
Different from existing graph pooling models~\cite{YYMRHL18,CVJKL18,LLK19,RST20,YJ20} which only coarsen graphs to generate graph representations, we aim to mutually utilise node-wise and graph-wise tasks to better encode multi-grained semantics into both node and graph representations under a unified framework.
We design a mechanism to allow the learned multi-grained semantics to enrich the node representations of the original graph $\mathcal{G}$ as shown in Figure~\ref{fig:architesture_AdamGNN}-(a).
Vice versa, the updated node representation can further ameliorate the graph representation in the next training iteration.
A reasonable unpooling operation that passes macro-level information to original nodes has not been well studied in the literature.
For instance, Gao et al.~\cite{GJ19} directly relocate the super node back into the original graph and utilise other GNN layers to spread its message to other nodes.
However, these additional aggregation operations cannot allow each node to receive meaningful information since some nodes may distant from super nodes, in such case these operations can exacerbate local-smoothing~\cite{LHW18}.

We implement the unpooling process 
by devising a \textit{top-down} message-passing mechanism,
which endows GNN models with meso/macro level knowledge.
Specifically, since $\mathbf{S}_{t}$ records how nodes of $\mathcal{G}_{t-1}$ form the super nodes of $\mathcal{G}_{t}$, 
so we utilise $\mathbf{S}_{t}$ to restore the generated ego-network representation at level $t$ to that at level $t\!-\!1$ until we arrive at the original graph $\mathcal{G}$, i.e., $t \to 0$, as follows:
\begin{equation}
    \hat{\mathbf{H}}_{t} = (\mathbf{H}_{t}^{\mathrm{T}} \mathbf{S}_{t}^{\mathrm{T}} \mathbf{S}_{t\!-\!1}^{\mathrm{T}} \dots \mathbf{S}_{1}^{\mathrm{T}})^{\mathrm{T}},
\end{equation}
where $\hat{\mathbf{H}}_{t} \in \mathbb{R}^{n \times d}$.
At the end of each iteration, nodes in the original graph $\mathcal{G}$ will receive high-level semantic messages from the different levels, i.e., $\{\hat{\mathbf{H}}_{1}, \dots , \hat{\mathbf{H}}_{T}\}$.
As illustrated in Figure~\ref{fig:architesture_AdamGNN}-(b)-(iv-i), the graph unpooling process can be treated as an inverse process of adaptive graph pooling process. 

\subsection{Flyback Aggregation}
\label{subsec:flyback_agg}
Since the super graphs at different granularity levels present multi-grained semantics and how each node utilises the received semantic information with its $flat$ representation is a challenging question. 
And nodes of the same graph may need different granularity levels' information. 
Therefore, we propose a novel attention mechanism to integrate the derived representations at different levels, given by:
\begin{equation}
    \mathbf{Z} = \mathbf{H} + \sum_{\forall t} \boldsymbol{\beta}_{t} \, \hat{\mathbf{H}}_{t},
\end{equation}
where the attention score $\boldsymbol{\beta}_{t}$ estimates the importance of the message from level $t$, given by:
$\boldsymbol{\beta}_{t}[i] =\frac{\mathrm{exp}(\overrightarrow{\mathbf{a}_{2}}^{\mathrm{T}} \, \sigma(\mathbf{W} \hat{\mathbf{H}}_{t}[i]) \parallel \mathbf{H}[i])))}{\sum_{j \in T} \mathrm{exp}(\overrightarrow{\mathbf{a}_{2}}^{\mathrm{T}} \, \sigma(\mathbf{W} \hat{\mathbf{H}}_{j}[i]) \parallel \mathbf{H}[i])))}$
where $\overrightarrow{\mathrm{\mathbf{a}}_{2}} \in \mathbb{R}^{2\pi}$ is the weight vector.
We term this process as the \textit{flyback} aggregation, which considers the attention scores of different levels and allows each node to decide whether/how to utilise semantic information from different granularity levels.
We will verify the effectiveness of \textit{flyback} aggregation in the ablation study of Section~\ref{subsec:experimental_evaluation_and_ablation_study} and discuss the explainability in Section~\ref{subsec:more_model_analysis}.

\subsection{Training Strategy}
\label{subsec:train_strategy}
Till now, there are still two challenges when training the model.
The first is how to highlight the difference among nodes' representations from different ego-networks.
Nodes belong to neighbouring ego-networks receive closely related messages from super nodes, since their super nodes are connected in the super graph, and local smoothing makes their representation vectors similar.
Representations of proximal nodes could be further closer to each other in the representation latent space.
To address this problem and enhance the discrimination capability between ego-networks, we exploit a self-optimisation strategy~\cite{XGF16}, which let nodes in different ego-networks distinguishable.
Specifically, we use the Student's \textit{t}-distribution ($Q$) as a kernel to measure the similarity between representation vectors of $v_{j}$ and ego $v_{i}$: $q_{ij}=\frac{(1+ \parallel \mathbf{Z}[j] - \mathbf{Z}[i] \parallel^{2}/\mu)^{-\frac{\mu+1}{2}}}{\sum_{i'}(1+ \parallel \mathbf{Z}[j] - \mathbf{Z}[i'] \parallel^{2}/\mu)^{-\frac{\mu+1}{2}}}$, 
where $v_{j}\in c_{\lambda}(v_{i})$, $v_{i'}$ are other ego nodes, $\mu$ are the degrees of freedom of Student's $t$-distribution.
Following this, $q_{ij}$ can be integrated the probability of assigning node $v_{j}$ to ego $v_{i}$.
In this paper, we set $\mu = 1$ the same as~\cite{XGF16}.
After, we propose to learn better node representations by matching $Q$ to the auxiliary target distribution ($P$), and we choose the proposition of~\cite{XGF16} which first raises $q_{i}$ to the second power and then normalises by frequency per ego-network:
$p_{ij}=\frac{q^{2}_{ij} / g_{i}}{\sum_{i'}(q^{2}_{\hat{i}j} / g_{i'})}$,
where $g_{i} = \sum_{j}q_{ij}$.
Therefore, apart from the task-related loss function $\mathcal{L}_{task}$,
we further define a {\it KL} divergence loss as:
\begin{equation}
    \mathcal{L}_{\it KL}={\it KL}(P \parallel Q)=\sum_{v_{j}} \sum_{v_{i}}p_{ij}\,\mbox{log}\frac{p_{ij}}{q_{ij}},
\end{equation}

The second challenge is to avoid the over-smoothing problem that nodes of a graph tend to have indistinguishable representations. 
GNN is proved as a special form of Laplacian smoothing~\cite{LHW18,CLLLZS20} that naturally assimilates the nearby nodes' representations.
AdamGNN will further exacerbate this problem because it distributes semantic information from one super node representation to all nodes of the ego network.
Therefore, we introduce the reconstruction loss, 
which can alleviate the over-smoothing issue and drive the node representations to retain the structure information of $\mathcal{G}$ by differentiating non-connected nodes' representations.
Specifically, the reconstruction loss is defined as:
\begin{equation}
    \mathcal{L}_{\it R} = -\frac{1}{n}\sum(\mathbf{A}_{ij} \cdot \mbox{log}(\mathbf{A}'_{ij})+(1 - \mathbf{A}_{ij})\cdot \mbox{log}(1 - \mathbf{A}'_{ij})),
\end{equation}
where $\mathbf{A}' = \mathrm{Sigmoid}(\mathbf{Z}^{\mathrm{T}} \mathbf{Z})$. 
Therefore, the overall loss function consists of the training task $\mathcal{L}_{\it Task}$, the self-optimising task $\mathcal{L}_{KL}$, and the reconstruction task $\mathcal{L}_{\it R}$, given by:
\begin{equation}
\label{eq-final}
    \mathcal{L} = \mathcal{L}_{\it Task} + \gamma\mathcal{L}_{\it KL} + \delta\mathcal{L}_{\it R},
\end{equation}
where $\mathcal{L}_{\it Task}$ is a flexible task-specific loss function,
and $\gamma$ and $\delta$ are two hyper-parameters that we will discuss in Section~\ref{subsec:experimental_setup}.
Note that for link prediction task we have $\mathcal{L} = \mathcal{L}_{\it R} + \gamma\mathcal{L}_{\it KL}$, since $\mathcal{L}_{\it Task}$ equals to $\mathcal{L}_{\it R}$.
Moreover, we will demonstrate the effectiveness of each component of loss function $\mathcal{L}$ in the ablation study of Section~\ref{subsec:experimental_evaluation_and_ablation_study}. 

\begin{algorithm}[!t]
\SetAlgoLined
\KwIn{
	graph $\mathcal{G}=(\mathcal{V}, \mathcal{E}, \mathbf{X})$. 
}
\KwOut{
	node representations $\mathbf{Z}$, graph representations $\mathbf{Z}_{g}$}
    $\mathbf{H} = \mathrm{ReLU}(\hat{\mathbf{D}}^{-\frac{1}{2}} \hat{\mathbf{A}} \hat{\mathbf{D}}^{\frac{1}{2}} \mathbf{X} \mathbf{W})$ \; 
    \For{$t \gets \{1, 2, \dots, T \}$}{
        \For{$v_{i} \gets \{v_{1}, v_{2}, \dots, v_{n} \}$}{
            \For{$v_{j} \in \mathcal{N}^{\lambda}_{i} $}{
                $\phi_{ij} = f^{non}_{\phi}(v_{i}, v_{j}) \times f^{lin}_{\phi}(v_{i}, v_{j})$;
            }
            $\phi_{i} = \frac{1}{\vert\mathcal{N}^{\lambda}_{i}\vert} \sum_{v_{j} \in \mathcal{N}^{\lambda}_{i}} \phi_{ij}$;
        }
        \For{$v_{i} \gets \{v_{1}, v_{1}, \dots, v_{n} \} $}{
            $\hat{N}_{p}=\{v_{i} \mid  \phi_{i} > \phi_{j}, \,  \forall v_{j} \in \mathcal{N}^{1}_{i}\}$;
        }
        $\hat{N}_{r}=\{v_{j} \mid v_{j} \notin c_{\lambda}(v_{i}), \forall v_{i} \in \hat{N}_{c} \}$ \;    
        Generate the super-node formation matrix: $\mathbf{S}_{t}$ \;
        \For{$v_{i} \in \hat{N}_{r}$}{
            $\mathbf{X}_{t}[i] = \mathbf{H}_{t\!-\!1}[i]$ ;
        }
        \For{$v_{i} \in \hat{N}_{p}$}{
            $\mathbf{X}_{t}[i] = \mathbf{H}_{t\!-\!1}[i] + \sum_{v_{j} \in c_{\lambda}(v_{i})\setminus v_{i}} \alpha_{ij} \mathbf{H}_{t\!-\!1}[j]$;
        }
        $\mathbf{A}_{t}= \mathbf{S}^{\mathrm{T}}_{t} \hat{\mathbf{A}}_{t\!-\!1} \mathbf{S}_{t}$ \;
        $\mathbf{H}_{t} = \mathrm{ReLU}(\hat{\mathbf{D}}^{-\frac{1}{2}}_{t} \hat{\mathbf{A}}_{t} \hat{\mathbf{D}}^{\frac{1}{2}}_{t} \mathbf{X}_{t} \mathbf{W}_{t})$ \;
       $\hat{\mathbf{H}}_{t} = (\mathbf{H}_{t}^{T} \mathbf{S}_{t}^{\mathrm{T}} \mathbf{S}_{t\!-\!1}^{\mathrm{T}} \dots \mathbf{S}_{1}^{\mathrm{T}})^{\mathrm{T}}$ \;
    }
    $\mathbf{Z} = \mathbf{H} + \sum^{T}_{t} \boldsymbol{\beta}_{t} \, \hat{\mathbf{H}}_{t}$ \;
    $\mathbf{Z}_{g} = \mathrm{READOUT}(\{\mathbf{Z}, \hat{\mathbf{H}}_{1}, \dots, \hat{\mathbf{H}}_{T} \}) \in \mathbb{R}^{d} $ \;
\caption{\textit{Adaptive Multi-grained GNNs}}
\label{alg:framework_AdamGNN}
\end{algorithm}

\subsection{Algorithm}
\label{subsec:appendix_algorithm}
We have presented the idea of AdamGNN and the design details of each component in Section~\ref{sec:proposed_approach}.
Given a graph $\mathcal{G}$, we first apply a primary GNN encoder to generate the primary node embedding (line $1$). 
Then we construct a multi-grained structure with $t$-th level (line $3$-$13$) with the proposed adaptive graph pooling operator. 
Meanwhile, we also propose a method to define the initial features of pooled super nodes (line $14$-$19$). 
The graph connectivity of the pooled graph is maintained by line $20$. 
We apply GNN encoder on the pooled graph to summarise the relationships between super nodes (line $21$) to learn macro grained semantics of $t$-th granularity level.
The learned multi-grained semantics will be further distributed to the original graph follows an unpooling operator (line $22$). 
Last, the \textit{flyback} aggregator generates the meso/macro level knowledge from different levels as the node representations of $\mathcal{G}$ (line $24$), and additional $\mathrm{READOUT}$ operators~\cite{CCBLV20} produce the node representations as to the graph representation (line $25$).

\smallskip\noindent
\textbf{Model scalability.}
According to the design for AdamGNN framework, we can find that the primary node representation learning module of each level and the adaptive graph pooling and unpooling operators are categorised as a local network algorithm~\cite{T16}, which only involves local exploration of the network structure. 
Therefore, our design enables AdamGNN to scale to representation learning on large-scale networks and to be friendly to distributed computing settings~\cite{QCDZYDWT20}. 
We present instances that utilise multi-GPU computing framework to accelerate the training process of AdamGNN in Section~\ref{subsec:more_model_analysis}. 

\section{Experiments} 
\label{sec:experiments}
\subsection{Experimental Setup}
\label{subsec:experimental_setup}
We evaluate our proposed model, AdamGNN, on $14$ benchmark datasets, and compare with $16$ competing methods over both node- and graph-wise tasks, including node classification, link prediction and graph classification.

\begin{table}[!ht]
\caption{Data statistics for node-wise tasks and the split for the semi-supervised node classification task.
N.A. means a dataset does not contain node attributes or does not support semi-supervised settings.
}
\label{table:summary_micro_datasets}
\centering
\resizebox{1.\linewidth}{!}{
\begin{tabular}{l | r | r | r | r | r}
            \hline
Dataset     & \#Nodes   & \#Edges   & \#Features    & \#Classes & \#Train-\#Val-\#Test \\
            \hline
ACM         & 3,025     & 13,128    & 1,870         & 3         & 60-500-1,000      \\
Citeseer    & 3,327     & 4,552     & 3,703         & 6         & 120-500-1,000      \\
Cora 	    & 2,708     & 5,278     & 1,433         & 7         & 140-500-1,000     \\	
DBLP        & 4,057     & 3,528     & 334           & 4         & 80-500-1,000      \\
Emails      & 799       & 10,182    & N.A.          & 18        & 180-309-310      \\
Pubmed      & 19,717    & 88,648    & 500           & 3         & 60-500-1,000      \\
Wiki        & 2,405     & 1,2178    & 4973          & 17        & 481-962-962    \\
Ogbn-arxiv  & 169,343   & 1,166,243	& 128           & 40        & N.A.    \\
\hline
\end{tabular}
}
\vspace{-3mm}
\end{table}

\begin{table}[!ht]
\caption{Data statistics for graph classification.
}
\label{table:summary_macro_datasets}
\centering
\resizebox{\linewidth}{!}{
\begin{tabular}{l | r | r | r | r | r}
                    \hline
Dataset         & \#Graphs  & \#Nodes (avg) & \#Edges (avg) & \#Features    & \#Classes    \\
                    \hline
NCI1            & 4,110     & 29.87         & 32.3          & 37            & 2             \\
NCI109          & 4,127     & 29.68         & 32.13         & 38            & 2             \\
D\&D            & 1,178     & 284.32        & 715.66        & 89            & 2             \\
MUTAG           & 188       & 17.93         & 19.79         & 7             & 2             \\
Mutagenicity    & 4,337     & 30.32         & 30.77         & 14            & 2             \\
PROTEINS        & 1,113     & 39.06         & 72.82         & 32            & 2             \\
\hline
\end{tabular}
}
\vspace{-3mm}
\end{table}

\smallskip\noindent
\textbf{Datasets}.
To validate the effectiveness of our model on real-world applications, we adopt datasets come from different domains with different topics and relations.
We use $8$ datasets for node-wise tasks (data statistics are summarised in Table~\ref{table:summary_micro_datasets}). 
Ogbn-arxiv~\cite{HFZDRLCL20}, ACM~\cite{BWSZLC20}, Cora~\cite{KW17}, Citeseer~\cite{KW17} and Pubmed~\cite{KW17} are paper citation graph datasets. 
DBLP~\cite{BWSZLC20} is an author graph dataset from the DBLP dataset. 
Emails~\cite{snapnets} is an email communication graph dataset. 
Wiki~\cite{YLZSC15} is a webpage graph dataset.

For the graph classification task, we adopt $6$ bioinformatics datasets~\cite{YJ20}
(data statistics are summarised in Table~\ref{table:summary_macro_datasets}).
D\&D and PROTEINS are dataset containing proteins as graphs.
NCI1 and NCI109 involve anticancer activity graphs.
The MUTAG and Mutagenicity consist of chemical compounds divided into two classes according to their mutagenic effect on a bacterium. 
Note that all datasets can be downloaded with our published code automatically. 

\smallskip\noindent
\textbf{Competing methods}.
For node-wise tasks, we adopt $10$ competing methods that include $7$ GNN models with flat message-passing mechanism, and $2$ state-of-the-art methods that contain a hierarchical structure:
MLP~\cite{H99}, 
GCN~\cite{KW17},
GraphSAGE~\cite{HYL17},
GAT~\cite{VCCRLB18},
GIN~\cite{XHLJ19},
PNA~\cite{CCBLV20},
\textsc{GCNII}~\cite{CWHDL20},
\textsc{GRAND}~\cite{FZDHLXYKT20},
GXN~\cite{LCZT20}
and \textsc{g-U-Net}~\cite{GJ19}.
%
For the graph classification task, except for GIN, PNA, GXN and \textsc{g-U-Net} which support graph-wise tasks, we adopt extra $7$ competing methods that involve the state-of-the-art models:
3WL-GNN~\cite{MHSL19},
\textsc{SortPool}~\cite{ZCNC18},
\textsc{DiffPool}~\cite{YYMRHL18},
\textsc{SagPool}~\cite{LLK19},
\textsc{EigentPool}~\cite{MWAT19},
\textsc{StructPool}~\cite{YJ20}
and ASAP~\cite{RST20}.
Note that we already carefully discussed these competing methods in Section~\ref{sec:related_work}; therefore, we do not repeat the method description to save page space. 
The competing model implementations can be found in our published project on Github.

\smallskip\noindent
\textbf{Evaluation settings}.
For the \textit{node-wise} tasks, we follow the \textit{supervised node classification} (Sup-NC) settings of PNA~\cite{CCBLV20}, i.e., using two sets of $10\%$ labelled nodes as validation and test sets, with the remaining $80\%$ labelled nodes used as the training set.
Meanwhile, we follow the \textit{semi-supervised node classification} (Semi-NC) settings of GCN~\cite{KW17}, and the data split is shown in Table~\ref{table:summary_micro_datasets}. 
That said, for Cora, Citeseer and Pubmed, we use the fixed splits and for other datasets, we randomly assign $20$ labelled nodes for each class for training, and $500$ and $1000$ nodes for validation and testing, respectively.
Note that since the Email does not have a sufficient number of nodes for classic Semi-NC setting, we choose $10$ labelled nodes for each class for training, and the rest data is evenly separated as validation and test sets.
Wiki is imbalanced, where some classes only have very few labelled nodes, e.g., class $12$ has $9$ labelled nodes and class $4$ has $10$ labelled nodes, which cannot support Semi-NC settings.
Therefore, we follow Sup-NC settings to split Wiki for the Semi-NC experimental parts but use only $20\%$ labelled nodes for training and the remaining nodes for validation and testing, respectively. 
Ogbn-arxiv follows the fixed split of OGB leaderboard~\cite{HFZDRLCL20}.
For the \textit{Link Prediction} (LP) task, we follow the settings of~\cite{YYL19}, i.e., using two sets of $10\%$ existing edges as validation and test sets, with the remaining $80\%$ edges used as the training set.
Note that, an equal number of non-existent links are randomly sampled and used for every set.
We present the average performance of $10$ times experiments with random seeds. 
The AUC score evaluates link prediction, and node classification tasks are evaluated by accuracy.
We conduct the experiments with random parameter initialisation with $10$ random seeds and report the average performance.

For the \textit{graph-wise} task, i.e., \textit{graph classification} (GC) task, we perform all experiments following the pooling pipeline of \textsc{SagPool}~\cite{LLK19}.
$80\%$ of the graphs are randomly selected as training, and the rest two $10\%$ graphs are used for validation and testing, respectively.
We conduct the experiments using $10$-fold cross-validation and report the average classification accuracy on $10$ random seeds.

\smallskip\noindent
\textbf{Model configuration}. 
For all methods, we set the embedding dimension $d=64$ and utilise the same learning rate $=0.01$, Adam optimiser, number of training epochs $=1000$ with early stop ($100$). 
In terms of the neural network layers, we report the one with better performance of GCNII with better performance among $\{8, 16, 32, 64, 128\}$; for other models, we report the one with better performance between $2\!-\!4$;
For all models with hierarchical structure (including AdamGNN), we use GCN as the GNN encoder for fair comparision. 
In terms of the number of levels that is required by hierarchical models, we present the one with better performance, between $2\!-\!5$.
On other hyper-parameter settings of competing methods, we employ the default values of each competing method as shown in the paper's official implementation. 
Particularly, for AdamGNN, by tuning the hyper-parameters based on the validation set, 
we have $\gamma = 0.1$ and $\delta = 0.01$ for Eq.~\ref{eq-final} for the experiments to let loss values lie in a reasonable range, i.e., $(0, 10)$. 
%
We employ Pytorch and PyTorch Geometric to implement all models.
Experiments were conducted with GPU (NVIDIA Tesla V100) machines.

\subsection{Experimental Evaluation and Ablation Study}
\label{subsec:experimental_evaluation_and_ablation_study}

\begin{table*}[!ht]
\caption{Results in accuracy (\%) for supervised and semi-supervised node classification on eight datasets.
\textsuperscript{\ddag} indicates the results from OGB leaderboard~\cite{HFZDRLCL20}.
The bold numbers represent the top-$2$ results.
}
\label{table:nc_results}
\centering
\resizebox{\linewidth}{!}{
\begin{tabular}{l | c  c | c  c | c  c | c  c | c  c | c  c | c  c | c }
\hline
Models          & \multicolumn{2}{c|}{ACM}  & \multicolumn{2}{c|}{Citeseer} & \multicolumn{2}{c|}{Cora} & \multicolumn{2}{c|}{Emails}   & \multicolumn{2}{c|}{DBLP} & \multicolumn{2}{c|}{Pubmed} & \multicolumn{2}{c}{Wiki} & Ogbn-arxiv                    \\ \hline
                                &Sup&Semi       &Sup&Semi       &Sup&Semi               &Sup&Semi       &Sup&Semi & Sup & Semi                                    & Sup & Semi & Sup                              \\ 
\hline
MLP~\cite{H99}                  &87.08&76.14    &70.87&63.27    &76.12&57.82            &N.A.&N.A.      &79.17&66.55     &83.41&71.83    &20.42&17.46 & $55.50^{\ddag}$                          \\
GCN~\cite{KW17}                 &92.25&86.14    &76.13&71.13    &88.90&81.50            &85.03&77.32    &82.68&72.22     &86.04&79.02    &57.36&46.19 & $71.74^{\ddag}$                          \\
GraphSAGE~\cite{HYL17}          &92.48&87.18    &76.75&71.19    &\textbf{88.92}&81.17            &85.80&78.19    &83.20&72.20     &86.01&79.14    &57.24&49.21 & $71.49^{\ddag}$                          \\
GAT~\cite{VCCRLB18} 	        &91.69&87.52    &\textbf{76.96}&72.43    &88.33&83.00   &84.67&77.35    &84.04&73.24     &86.21&79.01    &58.07&50.27 & 72.06                             \\	
GIN~\cite{XHLJ19}               &90.66&85.98    &76.39&70.27    &87.74&82.19            &87.18&79.23    &82.54&73.42     &87.10&79.11    &66.29&49.88 & 71.76                          \\
PNA~\cite{CCBLV20}              &\textbf{93.97}&\textbf{89.81}    &72.67&67.81    &88.78&79.54            &81.25&73.23     &\textbf{86.21}&72.40    &87.63&76.61    &21.67&19.58 & 72.37                           \\
GCNII~\cite{CWHDL20}              &93.05&87.80    &76.37&\textbf{73.37}    &88.07&\textbf{85.43}            &82.51&72.90     &84.48&72.82    &84.48&80.01    &60.24&\textbf{56.18} & $\textbf{72.74}^{\ddag}$                           \\
GRAND~\cite{FZDHLXYKT20}              &93.05&86.90    &76.88&\textbf{75.25}    &87.82&\textbf{85.41}            &52.53&71.94     &85.47&70.46    &86.63&\textbf{82.06}    &59.58&27.94 & 70.97                           \\
\hline
\textsc{g-U-Net}~\cite{GJ19}   &93.42&89.34    &75.59&72.97   &87.68&84.40           &\textbf{89.16}&\textbf{80.48}                     &85.27&\textbf{73.86}                            &\textbf{87.67}&79.06  &\textbf{71.33}&54.18 & 71.78                          \\
GXN~\cite{LCZT20}   &92.95&87.70    &75.93&71.07   &86.90&83.10           &88.68&77.91                     &84.66&71.27                            &86.02&78.01  &68.90&51.89 & 70.02                          \\
AdamGNN         &\textbf{94.37}&\textbf{90.72}    &\textbf{78.92}&73.03   &\textbf{90.92}& 84.66  &\textbf{91.88}&\textbf{83.23}   &\textbf{88.36}&\textbf{74.60}      &\textbf{89.81}&\textbf{80.48}   &\textbf{73.37}& \textbf{62.06} & \textbf{72.65}  \\
\hline
\end{tabular}
}
\vspace{-3mm}
\end{table*}

\smallskip\noindent
\textbf{Performance on node-wise tasks}.
We compare AdamGNN with $7$ GNN models and one pooling-based model, i.e., \textsc{g-U-Net},
since other pooling approaches do not provide an unpooling operator and thus cannot support node-wise tasks.
Results on node classification (with supervised and semi-supervised settings) are summarised in Table~\ref{table:nc_results}.
They show that AdamGNN can outperform most competing methods with up to $10.47\%$ and $5.39\%$ improvements on semi-supervised and supervised settings, respectively.
AdamGNN brings the most significant improvement in Wiki data with semi-supervised settings, and the competing method that only adopts node features, i.e., MLP, achieve terrible accuracy, $17.46\%$.
We argue that because the node features and node labels are weakly correlated in this dataset, multi-grained semantics provided by AdamGNN help to ameliorate the performance.

\begin{table*}[!ht]
\caption{Results in AUC for link prediction on seven datasets.
}
\label{table:lp_results}
\centering
\begin{tabular}{l | c | c | c | c | c | c | c}
            \hline
Models                          & ACM & Citeseer & Cora & Emails & DBLP & Pubmed & Wiki          \\ \hline
GCN~\cite{KW17}                 & 0.975 & 0.887 & 0.918 & 0.930 & 0.904 & 0.941  & 0.523         \\
GraphSAGE~\cite{HYL17}          & 0.972 & 0.884 & 0.908 & 0.923 & 0.889 & 0.905  & 0.577         \\
GAT~\cite{VCCRLB18}             & 0.968 & 0.910 & 0.912 & 0.930 & 0.889 & 0.947  & 0.594         \\	
GIN~\cite{XHLJ19}               & 0.787 & 0.808 & 0.878 & 0.859 & 0.820 & 0.927  & 0.501         \\
PNA~\cite{CCBLV20}              & 0.978 & 0.974 & 0.731 & 0.932 & 0.908 & 0.896  & 0.538        \\
GCNII~\cite{CWHDL20}            & 0.968 & 0.969 & 0.871 & 0.926 & 0.890 & 0.933  & 0.709      \\
\hline
\textsc{g-U-Net}~\cite{GJ19}   & 0.890 & 0.918     & 0.932 & 0.936     & 0.934 & 0.962  & 0.734         \\
GXN~\cite{LCZT20}   & 0.862 & 0.894     & 0.909 & 0.915     & 0.911 & 0.934  & 0.516         \\
AdamGNN                         & \textbf{0.988}    & \textbf{0.975}    & \textbf{0.948}    & \textbf{0.957}  & \textbf{0.965}  & \textbf{0.969} & \textbf{0.920}   \\
\hline
\end{tabular}
\vspace{-3mm}
\end{table*}

Link prediction results in Table~\ref{table:lp_results} show that AdamGNN can significantly outperform the $7$ competing methods by up to $25.3\%$ improvement in terms of AUC.
It indicates the versatility of AdamGNN on different node-wise tasks and exhibits the usefulness of modelling multi-grained semantics into node representations.
Similar to node classification task, AdamGNN again brings the most significant AUC improvement on the Wiki dataset, i.e., achieving $29.76\%$ improvement compared with the \textit{flat} GNN models.

\begin{table*}[!ht]
\caption{Results in accuracy (\%) for graph classification on six datasets.
}
\label{table:gc_results}
\centering
\begin{tabular}{l  c  c  c  c  c  c}
\hline
Models                              & NCI1              & NCI109            & D\&D              & MUTAG         & Mutagenicity        & PROTEINS     \\
\hline
GIN~\cite{XHLJ19}                   & 76.17$\pm$1.12    & 77.31$\pm$1.42    & 78.05$\pm$1.89    & 75.11$\pm$2.64 & 77.24$\pm$2.26                          & 75.37$\pm$1.62 \\
3WL-GNN~\cite{MHSL19}               & 79.38$\pm$1.73    & 78.34$\pm$1.90    & 78.32$\pm$2.44    & 78.34$\pm$3.39 & 81.52$\pm$2.23                          & 77.92$\pm$2.09 \\
PNA~\cite{CCBLV20}                  & 78.96$\pm$1.01    & 79.06$\pm$1.15    & 75.47$\pm$2.52    & \textbf{81.91}$\pm$2.59 & 81.72$\pm$1.46 & 77.72$\pm$2.25 \\
\hline
\textsc{SortPool}~\cite{ZCNC18}     & 72.25$\pm$1.33    & 73.21$\pm$2.21    & 73.31$\pm$2.43    & 71.47$\pm$2.31    & 74.65$\pm$3.35                          & 70.49$\pm$2.37 \\
\textsc{DiffPool}~\cite{YYMRHL18}   & 76.47$\pm$0.96    & 76.17$\pm$1.11    & 76.16$\pm$1.69    & 73.61$\pm$3.94    & 76.30$\pm$0.37                   & 71.90$\pm$2.75 \\
\textsc{g-U-Net}~\cite{GJ19}       & 77.56$\pm$1.92    & 77.02$\pm$2.30    & 73.98$\pm$2.63    & 76.60$\pm$5.03    & 78.64$\pm$3.11                   & 72.94$\pm$3.68 \\
\textsc{SagPool}~\cite{LLK19} 	    & 75.76$\pm$1.29    & 73.67$\pm$2.32    & 76.21$\pm$1.56    & 75.27$\pm$1.92    & 77.09$\pm$1.17                      & 75.27$\pm$0.57 \\	
\textsc{EigenPool}~\cite{MWAT19}    & 77.54$\pm$1.82    & 77.20$\pm$1.81    & 78.25$\pm$1.78    & 76.21$\pm$2.74    & 78.60$\pm$1.24                   & 75.19$\pm$1.95  \\
GXN~\cite{LCZT20}     & 74.18$\pm$2.20    & 71.39$\pm$1.90    & 74.92$\pm$2.04    & 74.10$\pm$4.36    & 75.53$\pm$1.83                   & 72.26$\pm$2.24 \\
\textsc{StructPool}~\cite{YJ20}     & 77.61$\pm$1.08    & 78.39$\pm$1.23    & 80.10$\pm$1.77    & 77.13$\pm$3.93    & 80.94$\pm$1.67                   & 78.84$\pm$1.70 \\
ASAP~\cite{RST20}                   & 78.21$\pm$1.75    & 78.16$\pm$1.62    & 79.50$\pm$1.80    & 80.17$\pm$1.77    & 81.52$\pm$1.24                   & \textbf{78.92$\pm$1.45} \\
AdamGNN                             & \textbf{79.77$\pm$1.29} & \textbf{79.36$\pm$1.03}    & \textbf{81.51$\pm$1.56} & 80.11$\pm$2.58   & \textbf{82.04$\pm$1.73}          & 77.04$\pm$0.78 \\
\hline
\end{tabular}
\vspace{-3mm}
\end{table*}

\smallskip\noindent
\textbf{Performance on graph-wise task}.
Experimental results are summarised in Table~\ref{table:gc_results}.
It is apparent that our AdamGNN achieves the best performance on $4$ of the $6$ datasets, and consistently outperforms most of competing pooling-based techniques by $1.76\%$ improvement.
For the datasets MUTAG and PROTEINS, our results are still competitive since PNA and ASAP only slightly outperform AdamGNN. 
Nevertheless, AdamGNN is still better than other baselines.
This is because our model involves adaptive pooling and unpooling operators to update node- and graph-wise information interactively, and further enhance the representations of nodes and graphs during the training process.

\begin{table}[!ht]
\caption{Comparison of AdamGNN with different loss functions on three tasks. 
NC task follows the supervised settings.
}
\label{table:ablation_analysis_loss_func}
\centering
\begin{tabular}{l  c  c  c  c}
\hline
             & DBLP  & Citeseer  & Mutagenicity     \\
            & (LP)    & (NC)        & (GC)                        \\
\hline
AdamGNN w/ $\mathcal{L}_{\it Task}$ & 0.956         & 76.63             & 79.04                     \\
AdamGNN w/ $\mathcal{L}_{\it Task}$+$\mathcal{L}_{\it KL}$   & -             & 77.17             & 78.94                     \\
AdamGNN w/ $\mathcal{L}_{\it Task}$+$\mathcal{L}_{\it R}$ 	& -             & 77.64             & 80.65                     \\	
AdamGNN (Full model)        & \textbf{0.965}  & \textbf{78.92}    & \textbf{82.04}            \\
\hline
\end{tabular}
\vspace{-3mm}
\end{table}

\smallskip\noindent
\textbf{Ablation study of different loss functions}.
The loss function of our AdamGNN consists of 
three parts, i.e., $\mathcal{L}_{\it Task}$, $\mathcal{L}_{\it R}$ and $\mathcal{L}_{\it KL}$.
We examine how each part contributes to the performance.
Table~\ref{table:ablation_analysis_loss_func} provides the results.
For the link prediction task, 
we have $\mathcal{L} = \mathcal{L}_{R} + \gamma\mathcal{L}_{\it KL}$,
since $\mathcal{L}_{\it Task}$ equals to $\mathcal{L}_{R}$.
Thus, two comparison experiments are missing in link prediction.
From the results, we can see that $\mathcal{L}_{R}$ can significantly improve the performance over three tasks. This is because it can eliminate the over-smoothing problem caused by the received messages from different granularity levels.
Meanwhile, $\mathcal{L}_{KL}$ can slightly improve the results of node-wise tasks as well.

\begin{table}[!ht]
\caption{Comparison of AdamGNN with and without \textit{flyback} aggregation in terms of graph classification accuracy on NCI1, NCI109 and Mutagenicity datasets.
}
\label{table:abaltion_analysis_flyback}
\centering
\begin{tabular}{l  c  c  c }
\hline
AdamGNN                         & NCI1          & NCI109   & Mutagenicity     \\
\hline
No \textit{flyback} aggregation & 75.54         & 77.49             & 79.89                     \\
Full model                      & \textbf{79.77}& \textbf{79.36}    & \textbf{82.04}            \\
\hline
\end{tabular}
\vspace{-3mm}
\end{table}

\smallskip\noindent
\textbf{Ablation study of the \textit{flyback} aggregation}.
Experimental results of node-wise tasks confirm that capturing multi-grained semantics in AdamGNN can help to learn more meaningful node representations.
Here, we further study whether \textit{flyback} aggregator can improve graph representations.
Specifically, we aim to see how the \textit{flyback} aggregator contributes to graph classification performance by removing and keeping it. 
The results are summarised in Table~\ref{table:abaltion_analysis_flyback}.
It is clear that the node representations enhanced by the \textit{flyback} aggregation can indeed improve the graph representation in the classification task. 

\begin{table}[!ht]
\caption{Comparison of AdamGNN with different number of granularity levels in terms of different tasks.
NC task follows the supervised settings.
}
\label{table:abaltion_analysis_levels}
\centering
\resizebox{\linewidth}{!}{
\begin{tabular}{l  c  c  c  c  c  c}
\hline
\# Levels ($T$)   & DBLP & Wiki     & ACM     & Citeseer   & Emails     & Mutagenicity     \\
\hline
                & LP            & LP                & NC                & NC                                & NC            & GC                        \\
\hline
1               & 0.951         & 0.912             & 92.60             & 77.68                             & 86.83         & 78.16                     \\
2               & 0.958         & 0.913             & 93.38             & 74.67                             & \textbf{91.88}& \textbf{82.04}            \\
3 	            & 0.959         & 0.917             & \textbf{94.37}    & 76.15                             & 90.61         & 81.58                     \\	
4               & \textbf{0.965}& \textbf{0.920}    & 90.84             & \textbf{78.92}                     & -             & 81.01                     \\
\hline
\end{tabular}
}
\vspace{-3mm}
\end{table}

\smallskip\noindent
\textbf{Ablation study of number of granularity levels}.
As it has been proved that the existing GNN models will have worse performance when the neural network goes deeper~\cite{LHW18},
here we examine how AdamGNN can be benefited from more granularity levels.
By varying the number of granularity levels, we report the performance of AdamGNN on link prediction, supervised node classification and graph classification, as summarised in Table~\ref{table:abaltion_analysis_levels}.
We can observe that increasing the number of granularity levels can improve both link prediction and node classifications. 
As for graph classification, $2$ levels would be a proper choice.

\begin{table}[!ht]
\caption{Comparison of AdamGNN with different primary GNN encoder, follow the semi-supervised node classification settings. 
}
\label{table:ablation_analysis_primary_encoder}
\centering
\begin{tabular}{l  c  c  c  c}
\hline
Models & Cora  & Citeseer  & Pubmed     \\
\hline
GCN                 & 81.50 & 71.13 & 79.02 \\
AdamGNN w/ GCN      & \textbf{84.66} & \textbf{73.03} & \textbf{80.48} \\
\hline
PNA                 & 79.54 & 67.81 & 76.61 \\
AdamGNN w/ PNA      & \textbf{81.21} & \textbf{70.90} & \textbf{78.73} \\
\hline
GCNII 	            & 85.43 & 73.37 & 80.01 \\	
AdamGNN w/ GCNII 	& \textbf{85.81} & \textbf{74.13} & \textbf{81.37} \\	
\hline
\end{tabular}
\vspace{-3mm}
\end{table}

\smallskip\noindent
\textbf{Ablation study of different primary GNN encoders}.
We adopted GCN as the default primary GNN encoder in model presentation (Section~\ref{sec:proposed_approach}) and previous experiments. 
Here, we present more experimental results by endowing AdamGNN with advanced GNN encoders in Table~\ref{table:ablation_analysis_primary_encoder}. 
The table demonstrates that advanced GNN encoders can still benefit from the multi-grained semantics of AdamGNN. 
For instance, GCNII can stack lots of layers to capture long-range information; however, it still follows a \textit{flat} message-passing mechanism hence naturally ignoring the multi-grained semantics.
AdamGNN further ameliorates this problem for better performance.

\begin{figure}[!ht]
\centering
\includegraphics[width=.98\linewidth]{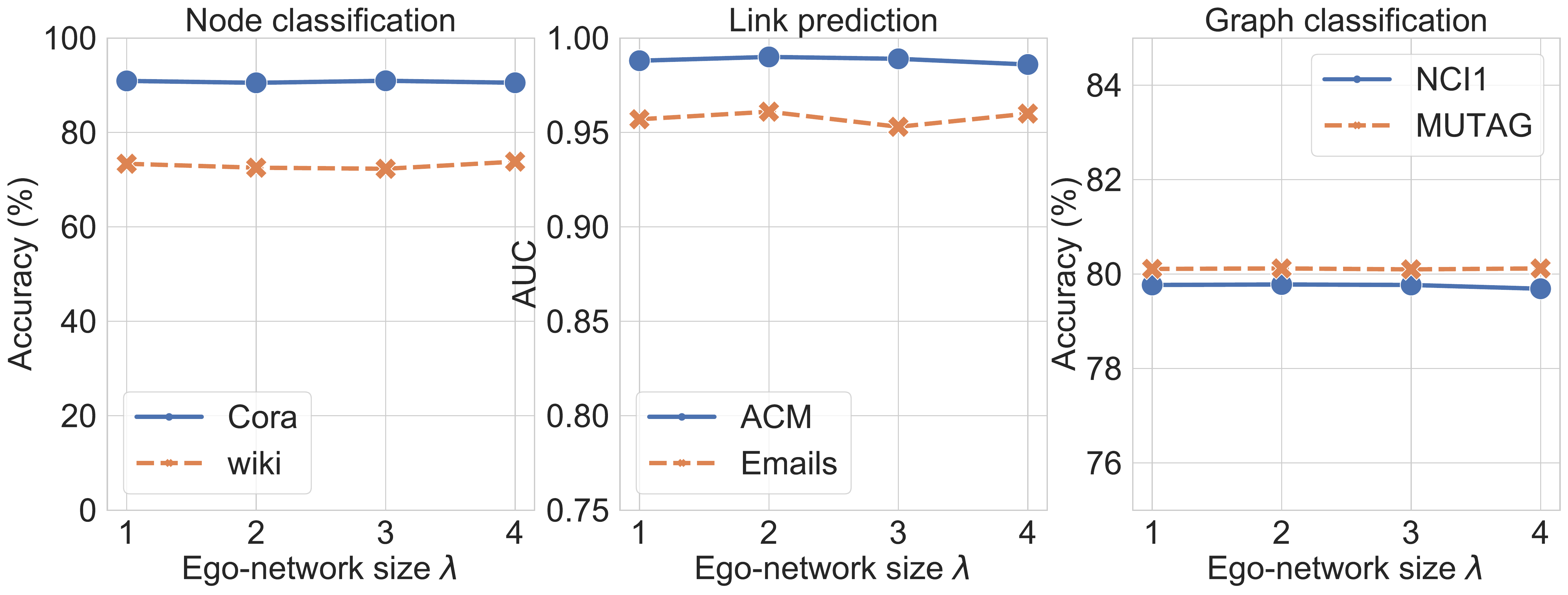}
\caption{
Ablation study of Ego-network size $\lambda$ in terms of different tasks. 
NC task follow the supervised settings. 
}
\label{fig:different_lambda}
\vspace{-3mm}
\end{figure}

\smallskip\noindent
\textbf{Ablation study of Ego-network size ($\lambda$)}.
The size of an ego-network as defined in Section~\ref{sec:proposed_approach} is captured by $\lambda$. 
We present an ablation study to investigate the influence of $\lambda$ on AdamGNN's performance, results are summarised in Figure~\ref{fig:different_lambda}. 
The figure indicates that $\lambda$ has no significant influence on the model performance. 
We simply adopt $\lambda\!=\!1$ throughout the paper.

\subsection{More Model Analysis}
\label{subsec:more_model_analysis}

\noindent
\textbf{Running time comparison}.
We present the average epoch training time of different node and graph classification models in Table~\ref{table:nc_running_time} and Table~\ref{table:gc_running_time}, respectively.
In terms of node classification task, AdamGNN requires more training time due to the computation cost of $\alpha_{ij}$ and $\boldsymbol{\beta}_{t}$, similar to any attention-mechanism enhanced models~\cite{VSPUJGKP17}. 
However, AdamGNN is designed as a local network algorithm, maintaining good scalability; hence it can be easily accelerated by mini-batch and multi-GPUs computing frameworks~\cite{KSOISLC21}.
It will significantly mitigate the computational issues. 
On the other hand, AdamGNN follows the sparse design similar to \textsc{SagPool}, ASAP, striking a balance between performance improvement and maintaining proper time efficiency. 
\textsc{DiffPool} and \textsc{StructPool} employ a \textit{dense} mechanism that is not easily scalable to larger graphs~\cite{CVJKL18}, and \textsc{g-U-Net} uses convolution operations, which bring additional computation cost, to distribute the received information to the graph. 

\begin{table}[!ht]
\caption{Average one epoch running time (in seconds) for supervised node classification task.
($\times$2) means accelerated by 2 GPUs~\cite{KSOISLC21}. 
}
\label{table:nc_running_time}
\centering
\begin{tabular}{l  c  c  c }
\hline
Models              & ACM    & Citeseer   & Cora  \\
\hline
GCN                 & 0.008  & 0.008      & 0.009      \\
GAT                 & 0.010  & 0.011      & 0.011      \\
GCNII               & 0.045  & 0.040      & 0.041      \\
\textsc{g-U-Net}   & 0.076  & 0.064      & 0.069      \\
\hline
AdamGNN ($\times$1) & 0.087  & 0.072      & 0.074      \\
AdamGNN ($\times$2) & 0.059  & 0.050      & 0.052      \\
AdamGNN ($\times$3) & 0.050  & 0.047      & 0.048      \\
\hline
\end{tabular}
\vspace{-3mm}
\end{table}

\begin{table}[!ht]
\caption{Average one epoch running time (in seconds) for graph classification task.
}
\label{table:gc_running_time}
\centering
\begin{tabular}{l  c  c  c }
\hline
Models              & NCI1  & NCI109    & PROTEINS  \\
\hline
\textsc{DiffPool}   & 6.23  & 3.22      & 3.65      \\
\textsc{SagPool}    & 1.95  & 1.55      & 0.45      \\
\textsc{g-U-Net}   & 4.58  & 4.45      & 1.46      \\
\textsc{EigenPool}  & 3.88  & 3.54      & 1.38      \\
ASAP                & 2.04  & 1.83      & 1.09      \\
\textsc{StructPool} & 6.31  & 6.04      & 1.34      \\
\hline
AdamGNN             & 3.62  & 3.24      & 1.03      \\
\hline
\end{tabular}
\vspace{-3mm}
\end{table}

\begin{figure}[!ht]
\centering
\includegraphics[width=.95\linewidth]{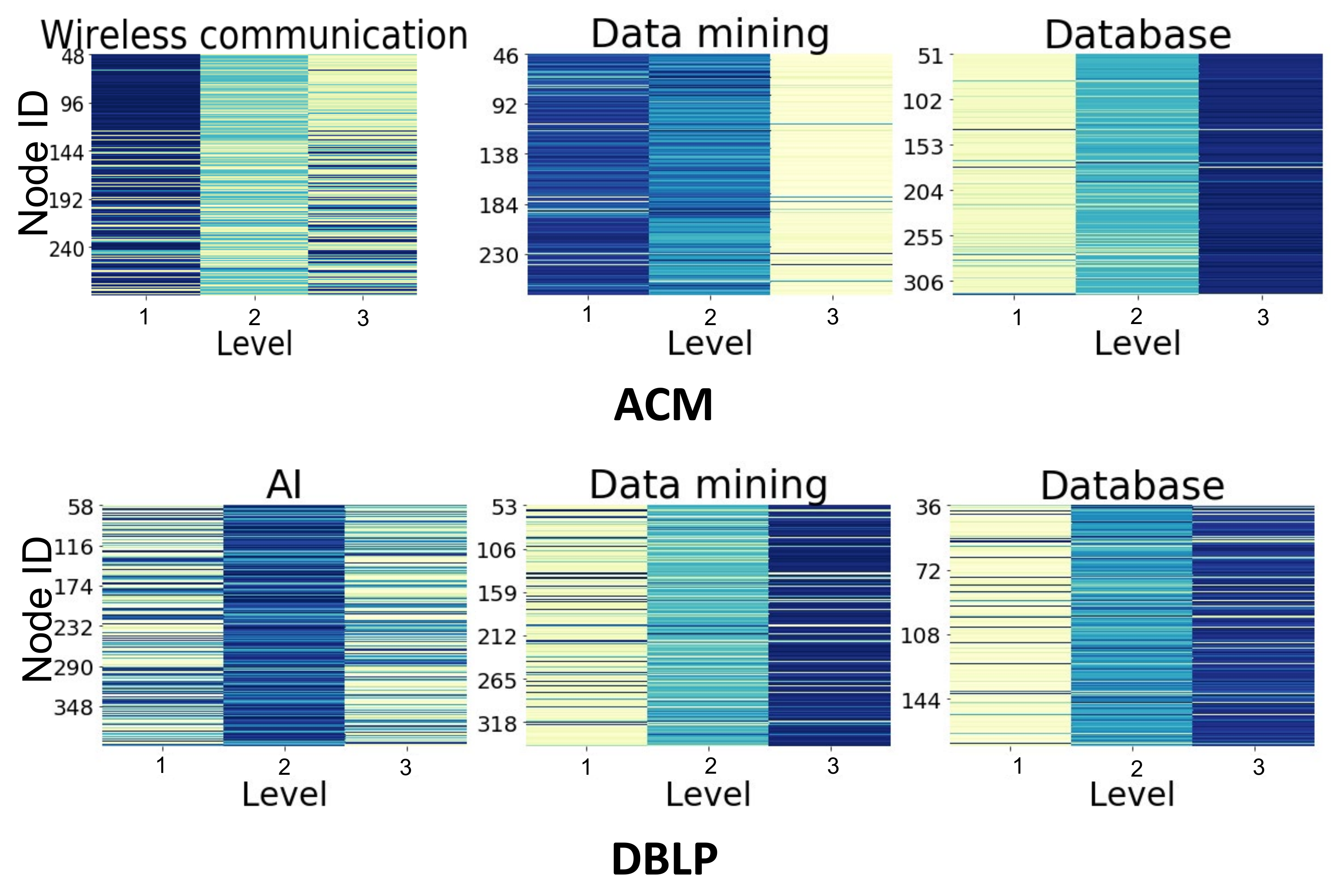}
\caption{
Visualisation of attention weight for messages at different granularity levels.
Dark colours indicate higher weights. 
}
\label{fig:attention_analysis}
\vspace{-3mm}
\end{figure}

\smallskip\noindent
\textbf{Messages from different levels}.
We aim to figure out the importance of received messages from different granularity levels since messages from different levels contain the meso/macro-level knowledge encoded by super nodes.
Here we consider the node classification on the ACM and DBLP as an example.
Specifically, ACM's paper nodes are labelled with $3$ topics: database (DB), data mining (DM) and wireless communication (WC);
DBLP's author nodes have $4$ research areas: AI, DB, DM and computer vision.
The node classification task on these two datasets is predicting the paper/scholar's research area.
The attention scores of nodes that highlight the importance of different levels' messages are plotted in Figure~\ref{fig:attention_analysis}.
We can find different distributions of attention weights over different granularity levels for various areas' classifications.
The relatively general topics, i.e., AI and WC, receive messages from different levels with relatively indistinguishable weights, i.e., higher attention scores of nodes are distributed across levels.
The DM topic in two datasets has different attention patterns: it receives messages from level $1$ with the greatest attention in ACM but receives greatest attention messages from level $3$ in DBLP.
This is because DM is not closely related to the other two topics of ACM dataset,
Scholars of DM-related papers are less possible collaborate with researchers from DB or WC.
DM papers are close to each other in the network. 
Thus DM papers only need to receive level $1$ granularity semantic information summarised from neighbouring nodes.
In contrast, DBLP's other $3$ topics are close to DM.
DM-related scholars may cite any other scholars' papers, and information related to DM is scattered over author nodes in DBLP network.
Therefore, DM researchers tend to be characterised by level $3$ semantics from a wide range.

\begin{figure}[!ht]
\centering
\includegraphics[width=.95\linewidth]{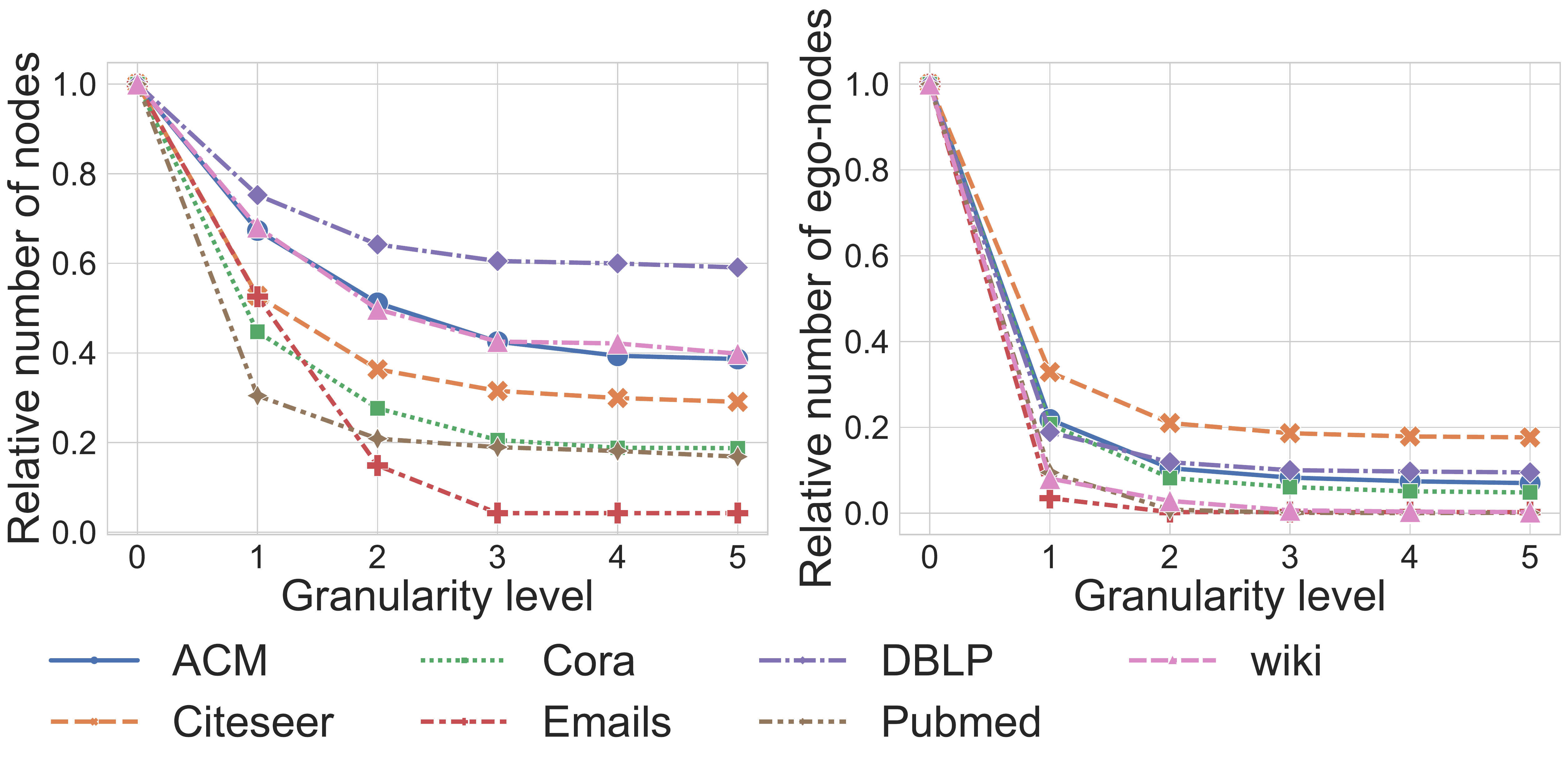}
\caption{
Visualisation of different granularity levels.
}
\label{fig:different_levels}
\vspace{-3mm}
\end{figure}

\smallskip\noindent
\textbf{Visualisation of different granularity levels}. 
To better understand the process of learning multi-grained semantics, we report the relative numbers of nodes (i.e., node ratio concerning the original graph) at different granularity levels generated by AdamGNN in $7$ datasets, as shown in Figure~\ref{fig:different_levels}.
In particular, we set the max number of granularity levels as $5$, and level $0$ indicates the original graph.
We train AdamGNN with semi-supervised node classification and report the relative number of nodes and selected ego-nodes at each level.
In the right, we can find out that the number of ego-nodes can stay stable after $1\!-\!2$ times of our adaptive pooling which indicates that AdamGNN can effectively find a compact structure that contains multi-grained semantics.
In the left, we can see that the number of nodes of each level stabilises after $3\!-\!4$ times of our adaptive pooling which illustrates AdamGNN will maintain the graph size at a proper level to avoid dense super graph.

\begin{figure}[!ht]
\centering
\includegraphics[width=.85\linewidth]{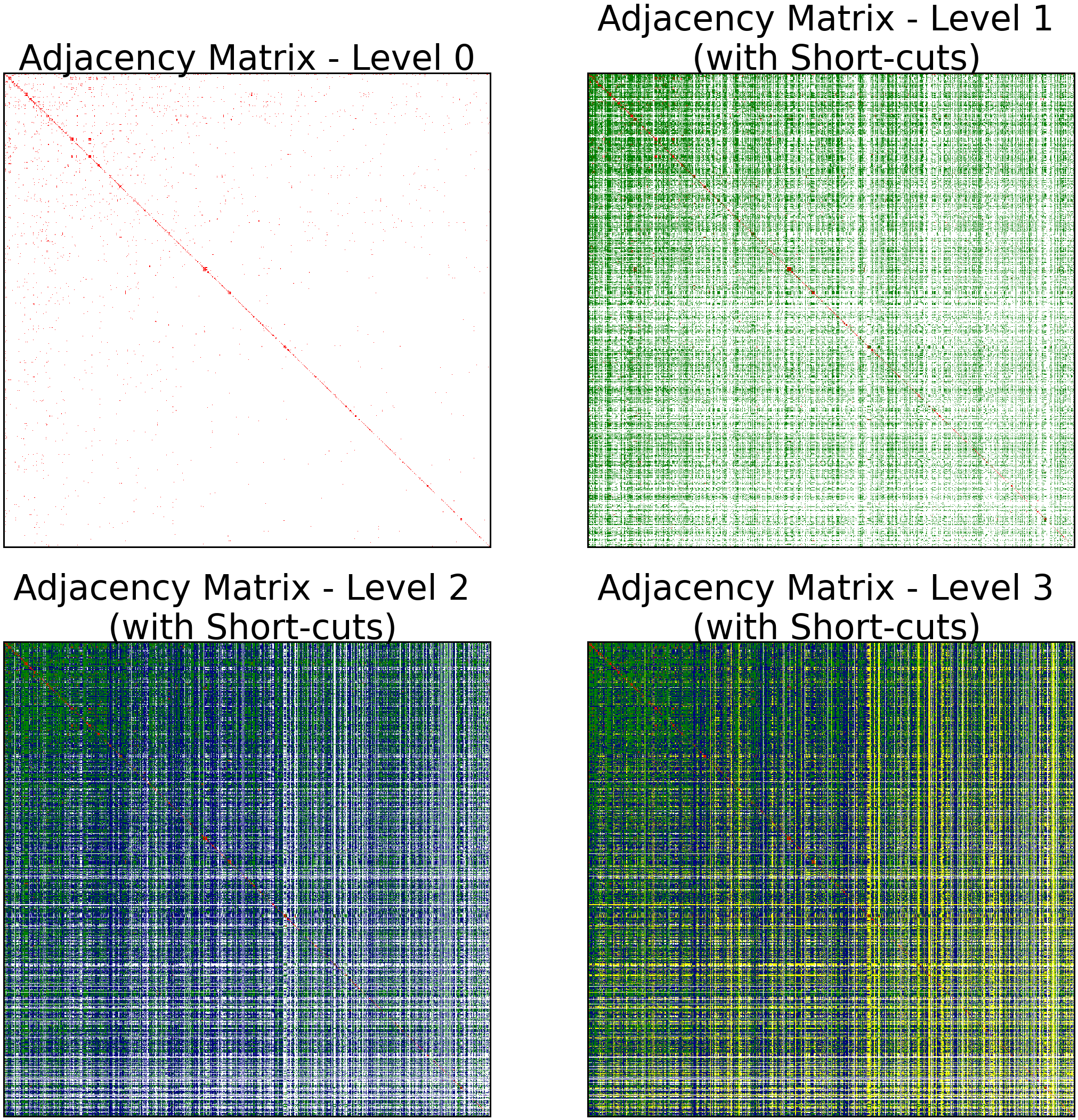}
\caption{
Visualisation of the adjacency matrices stacking the $1$-st (green), $2$-nd (blue), and $3$-rd (yellow) granularity levels on the Wiki dataset. 
}
\label{fig:diff_adj}
\vspace{-3mm}
\end{figure}

\smallskip\noindent
\textbf{Visualisation of adjacency matrices at different granularity levels}.
One of the fundamental limitations of existing GNNs is the inability of capturing 
long-range node interactions in the graph~\cite{LMQDAG19}.
We find that AdamGNN can provide a possible solution to overcome this limitation.
AdamGNN allows nodes to receive messages from far-away nodes with the support of the adaptive multi-grained structure. 
That said, the learned multi-grained structure can be regarded as a kind of \textit{short-cuts} to let far-away nodes be aware of each other.
We visualise these short-cut connections at different levels on the Wiki dataset in Figure~\ref{fig:diff_adj}.
Figure-level $0$ plots the original adjacency matrix, Figure-level $1$ exhibits the learned short-cuts by the first pooled super graph (in green).
Similarly, figure-levels $2$ and $3$ present the derived short-cuts by further stacking the adjacency matrices of the second (in blue) and third (in yellow) pooled graphs, respectively.
We can clearly see that the original graph of the Wiki dataset is very sparse, and AdamGNN adds short-cuts between nodes with the help of the learned multi-grained structure. In this way,
AdamGNN allows nodes to capture global information with few adaptively pooled graphs.

\begin{figure}[!ht]
\centering
\includegraphics[width=.9\linewidth]{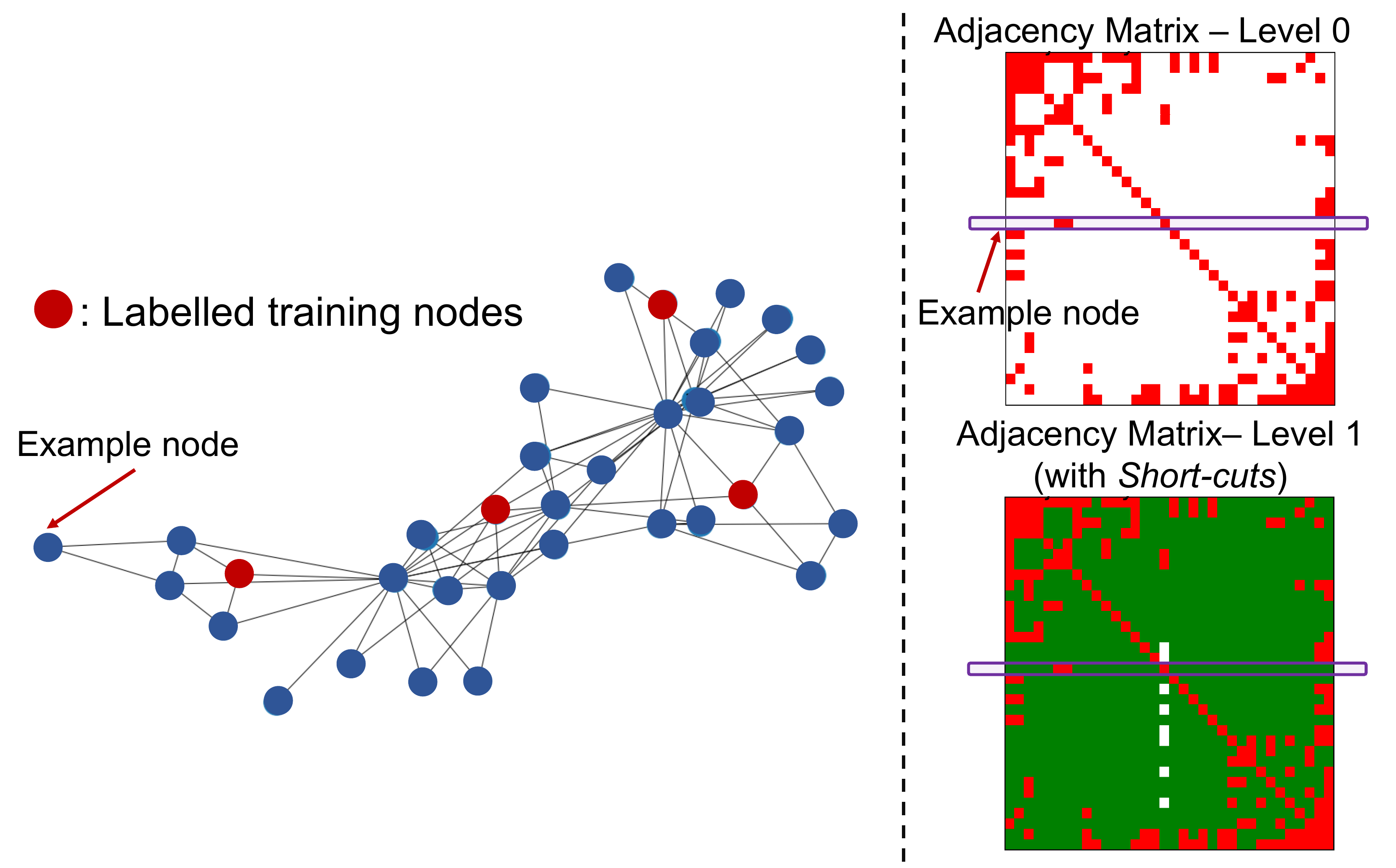}
\caption{
Visualisation of network structure and adjacency matrix at different granularity levels of \textit{Karate}-club dataset
}
\label{fig:karate_analysis}
\vspace{-4mm}
\end{figure}

\begin{table}[!ht]
\caption{Performance comparison for $1$-shot-NC task on \textit{Karate}-club dataset.
}
\label{table:karate_results}
\centering
\begin{tabular}{l  c  c }
\hline
Model   & $1$-Layer & $2$-Layers \\
\hline
GCN                     & 65.26         & 69.74     \\
AdamGNN ($1$-level)     & 88.65         & 97.91     \\
\hline
\end{tabular}
\vspace{-3mm}
\end{table}

\smallskip\noindent
\textbf{Exploration of \textit{short-cuts} of AdamGNN}.
To explore the \textit{short-cuts} derived by AdamGNN, we perform another empirical analysis on one additional network, i.e., the \textit{Karate}~\cite{Z77} club network.
We choose $1$-shot NC as target task, where we randomly select one sample from each class as training set, an equal number nodes as validation set and the rest nodes for test.
Network structure, experimental results is summarised in Table~\ref{table:karate_results} and two adjacency matrices are shown in Figure~\ref{fig:karate_analysis}. 
Short-cuts derived by the first pooled super graph are depicted in green in the level 1 adjacency matrix.
We find that AdamGNN outperforms GCN on $1$-shot NC task with up to $40.5\%$ performance improvements.
The two figures in the right part of Figure~\ref{fig:karate_analysis} clearly demonstrate that the \textit{short-cuts} derived by AdamGNN make the example node aware of far-away nodes.

\begin{figure}[!ht]
\centering
\includegraphics[width=.95\linewidth]{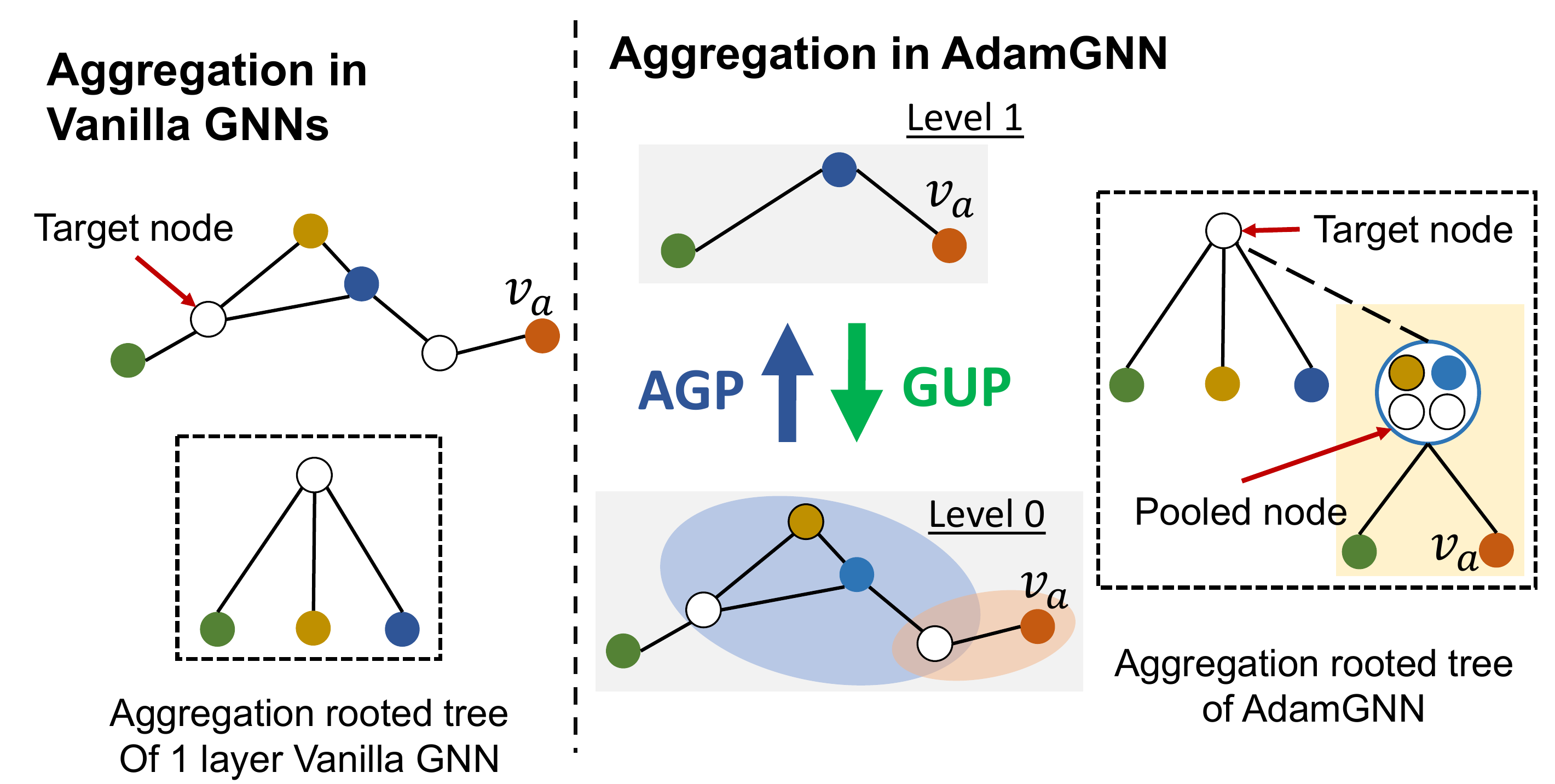}
\caption{
A toy example of the aggregation schemas of vanilla GNNs (left) and AdamGNN (right).
}
\label{fig:hierarcal_mp}
\vspace{-3mm}
\end{figure}

\smallskip\noindent
\textbf{Comparing aggregation mechanisms}.
To demonstrate the internal aggregation mechanism of AdamGNN and figure out the reason that it leads to performance improvements as shown in Section~\ref{subsec:experimental_evaluation_and_ablation_study}, we give a toy example of the aggregation schemas in vanilla GNNs and AdamGNN.
As shown in Figure~\ref{fig:hierarcal_mp}, $1$-layer vanilla GNN can only capture limited information as presented in the left rooted tree.
Nevertheless, thanks to the adaptive hierarchical structure learned by AdamGNN, target nodes can receive multi-grained semantics
as well as endowed with the ability to capture information from nodes from a long-range.
For instance, node $v_{a}$'s message cannot be obtained by target node with a few-layers vanilla GNN, but AdamGNN allows the target node to receive $v_{a}$'s information with $1$ granularity level.
The level $1$'s graph allows the super node to receive a message from node $v_{a}$ and pass it to the target nodes by \textit{flyback} aggregator.



\section{Conclusion} 
\label{sec:conclusion}
To summarise, we proposed AdamGNN, a method that integrates multi-grained semantics into node representations and realises collective optimisation between node- and graph-wise tasks in one unified process. 
We have designed an adaptive and efficient pooling operator with a novel ego-network selection approach to encode the multi-grained structural semantics, and a training strategy to overcome the over-smoothing problem. 
Extensive experiments conducted on $14$ real-world datasets showed the promising effectiveness of AdamGNN on node- and graph-wise downstream tasks.
One future direction is to appropriately apply the adaptive multi-grained structure on heterogeneous networks for node and graph level tasks.


\ifCLASSOPTIONcompsoc
  \section*{Acknowledgments}
\else
  \section*{Acknowledgment}
\fi
This work is supported by the Luxembourg National Research Fund 
through grant PRIDE15/10621687/SPsquared, and supported by Ministry of Science and Technology (MOST) of Taiwan under grants 110-2221-E-006-136-MY3, 110-2221-E-006-001, and 110-2634-F-002-051.

\bibliography{normal_generated.bib}
\bibliographystyle{IEEEtranS}

\vspace{-3mm}
\begin{IEEEbiography}[{\includegraphics[width=1in,height=1.25in,clip,keepaspectratio]{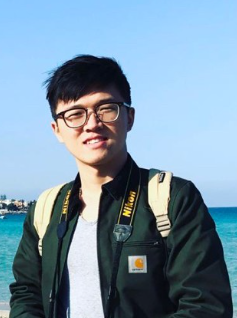}}]%
{Zhiqiang Zhong} is currently pursuing his Ph.D. degree at Faculty of Science, Technology and Medicine, University of Luxembourg, Luxembourg.
His principal research interest is in graph neural networks
and its applications in urban areas.
\end{IEEEbiography}

\vspace{-3mm}
\begin{IEEEbiography}[{\includegraphics[width=1in,height=1.25in,clip,keepaspectratio]{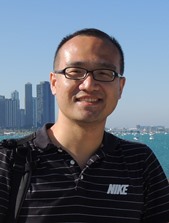}}]%
{Cheng-Te Li} is an Associate Professor at Institute of Data Science and Department of Statistics, National Cheng Kung University (NCKU), Tainan, Taiwan. He received my Ph.D. degree (2013) from Graduate Institute of Networking and Multimedia, National Taiwan University. Before joining NCKU, He was an Assistant Research Fellow (2014-2016) at CITI, Academia Sinica. Dr. Li's research targets at Machine Learning, Deep Learning, Data Mining, Social Networks and Social Media Analysis, Recommender Systems, and Natural Language Processing. He has a number of papers published at top conferences, including KDD, WWW, ICDM, CIKM, SIGIR, IJCAI, ACL, ICDM, and ACM-MM. He leads Networked Artificial Intelligence Laboratory (NetAI Lab) at NCKU.
\end{IEEEbiography}

\vspace{-3mm}
\begin{IEEEbiography}[{\includegraphics[width=1in,height=1.25in,clip,keepaspectratio]{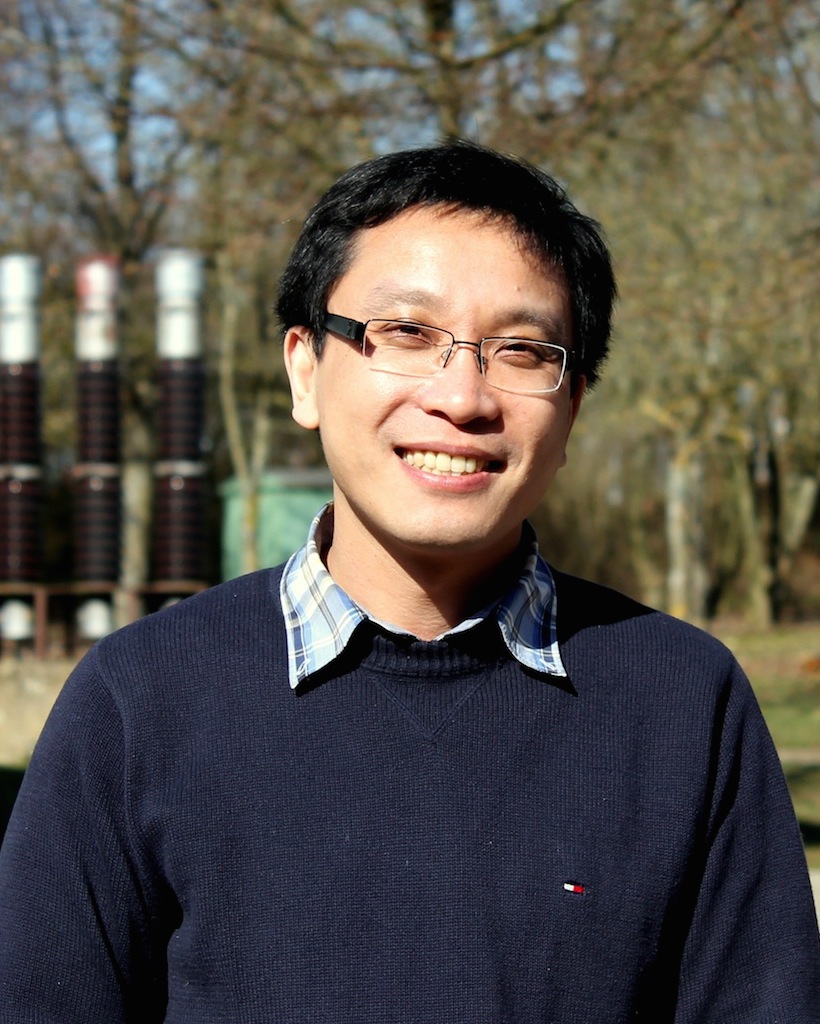}}]%
{Jun Pang} is currently a Senior Researcher at Faculty of Science, Technology and Medicine \& Interdisciplinary Centre for Security, Reliability and Trust, University of Luxembourg, Luxembourg. He received his Ph.D. degree from Vrije Universiteit Amsterdam in 2004. His research interests include formal methods, social media mining, computational systems biology, security and privacy. 
\end{IEEEbiography}




\end{document}